\documentclass[a4paper,11pt]{article}
\pdfoutput=1
\usepackage{a4wide}
\usepackage{amsmath,amsfonts,amssymb,amsthm}
\usepackage[colorlinks,citecolor=blue]{hyperref}
\usepackage{enumitem}
\usepackage{bbm}
\usepackage{microtype}
\usepackage{multirow}

\newtheorem{theorem}{Theorem}[section]
\newtheorem{lemma}[theorem]{Lemma}
\newtheorem{proposition}[theorem]{Proposition}
\newtheorem{corollary}[theorem]{Corollary}

\theoremstyle{definition}
\newtheorem{definition}[theorem]{Definition}

\newtheorem{remark}[theorem]{Remark}

\DeclareMathOperator*{\argmin}{argmin}

\numberwithin{equation}{section}
\numberwithin{table}{section}

\def \bB {\mathbb{B}}

\def \bD {\mathbb{D}}
\def \bE {\mathbb{E}}

\def \bN {\mathbb{N}}

\def \bR {\mathbb{R}}
\def \bS {\mathbb{S}}

\def \bY {\mathbb{Y}}
\def \bZ {\mathbb{Z}}

\def \cD {\mathcal{D}}
\def \cE {\mathcal{E}}
\def \cF {\mathcal{F}}
\def \cG {\mathcal{G}}
\def \cH {\mathcal{H}}

\def \cL {\mathcal{L}}

\def \cN {\mathcal{N}}
\def \cO {\mathcal{O}}
\def \cP {\mathcal{P}}

\def \cR {\mathcal{R}}

\def \cT {\mathcal{T}}

\def \cW {\mathcal{W}}

\def \NN {\mathcal{NN}}
\def \CNN {\mathcal{CNN}}

\def \sgn {\,{\rm sgn}\,}
\def \Pdim {\,{\rm Pdim}\,}

\makeatletter
\def\blfootnote{\xdef\@thefnmark{}\@footnotetext}
\makeatother

\begin{document}
\title{Optimal rates of approximation by shallow ReLU$^k$ neural networks and applications to nonparametric regression}
\author{
Yunfei Yang \thanks{Department of Mathematics, City University of Hong Kong, Kowloon, Hong Kong, China. (E-mail: yunfyang@cityu.edu.hk)}
\and
Ding-Xuan Zhou \thanks{School of Mathematics and Statistics, University of Sydney, Sydney, NSW 2006, Australia. (E-mail: dingxuan.zhou@sydney.edu.au)}
}
\date{}
\maketitle

\begin{abstract}
We study the approximation capacity of some variation spaces corresponding to shallow ReLU$^k$ neural networks. It is shown that sufficiently smooth functions are contained in these spaces with finite variation norms. For functions with less smoothness, the approximation rates in terms of the variation norm are established. Using these results, we are able to prove the optimal approximation rates in terms of the number of neurons for shallow ReLU$^k$ neural networks. It is also shown how these results can be used to derive approximation bounds for deep neural networks and convolutional neural networks (CNNs). As applications, we study convergence rates for nonparametric regression using three ReLU neural network models: shallow neural network, over-parameterized neural network, and CNN. In particular, we show that shallow neural networks can achieve the minimax optimal rates for learning H\"older functions, which complements recent results for deep neural networks. It is also proven that over-parameterized (deep or shallow) neural networks can achieve nearly optimal rates for nonparametric regression.


\end{abstract}

\section{Introduction}

Neural networks generate very popular function classes used in machine learning algorithms \cite{anthony2009neural,lecun2015deep}. The fundamental building blocks of neural networks are ridge functions (also called neurons) of the form $x \in \bR^d \mapsto \rho((x^\intercal,1)v)$, where $\rho:\bR\to \bR$ is a continuous activation function and $v\in \bR^{d+1}$ is a trainable parameter. It is well-known that a shallow neural network with non-polynomial activation
\begin{equation}\label{shallow NN}
f(x) = \sum_{i=1}^N a_i \rho((x^\intercal,1)v_i),
\end{equation}
is universal in the sense that it can approximate any continuous functions on any compact set with desired accuracy when the number of neurons $N$ is sufficiently large \cite{cybenko1989approximation,hornik1991approximation,pinkus1999approximation}. The approximation and statistical properties of neural networks with different architectures have also been widely studied in the literature \cite{yarotsky2017error,zhou2020universality,bartlett2019nearly,schmidthieber2020nonparametric}, especially when $\rho$ is a sigmoidal activation or $\rho$ is the ReLU$^k$ function $\sigma_k(t) = \max\{0,t\}^k$, the $k$-power of the rectified linear unit (ReLU) with $k\in \bN_0:=\bN\cup \{0\}$.

The main focus of this paper is rates of approximation by neural networks. For classical smooth function classes, such as H\"older functions, Mhaskar \cite{mhaskar1996neural} (see also \cite[Theorem 6.8]{pinkus1999approximation}) presented approximation rates for shallow neural networks, when the activation function $\rho \in C^\infty(\Omega)$ is not a polynomial on some open interval $\Omega$ (ReLU$^k$ does not satisfy this condition). It is known that the rates obtained by Mhaskar are optimal if the network weights are required to be continuous functions of the target function. Recently, optimal rates of approximation have also been established for deep ReLU neural networks \cite{yarotsky2017error,yarotsky2018optimal,shen2020deep,lu2021deep}, even without the continuity requirement on the network weights. All these approximation rates are obtained by using the idea that one can construct neural networks to approximate polynomials efficiently. There is another line of works \cite{barron1993universal,makovoz1996random,klusowski2018approximation,siegel2020approximation,siegel2022sharp} studying the approximation rates for functions of certain integral forms (such as (\ref{infinite width NN})) by using a random sampling argument due to Maurey \cite{pisier1981remarques}. In particular, Barron \cite{barron1993universal} derived dimension independent approximation rates for sigmoid type activations and functions $h$, whose Fourier transform $\widehat{h}$ satisfies $\int_{\bR^d} |\omega| |\widehat{h}(\omega)| d\omega<\infty$. This result has been improved and generalized to ReLU activation in recent articles \cite{klusowski2018approximation,siegel2020approximation,siegel2022sharp}.

In this paper, we continue the study of these two lines of approximation theories for neural networks (i.e. the constructive approximation of smooth functions and the random approximation of integral representations). Our main result shows how well integral representations corresponding to ReLU$^k$ neural networks can approximate smooth functions. By combining this result with the random approximation theory of integral forms, we are able to establish the optimal rates of approximation for shallow ReLU$^k$ neural networks. Specifically, we consider the following function class defined on the unit ball $\bB^d$ of $\bR^d$ induced by vectors on unit sphere $\bS^d$ of $\bR^{d+1}$ as
\begin{equation}\label{infinite width NN}
\cF_{\sigma_k}(M) := \left\{ f(x) = \int_{\bS^d} \sigma_k((x^\intercal,1)v) d\mu(v): \|\mu\| \le M, x\in \bB^d \right\},
\end{equation}
which can be regarded as a shallow ReLU$^k$ neural network with infinite width \cite{bach2017breaking}. The restriction on the total variation $\|\mu\| :=|\mu|(\bS^d) \le M$ gives a constraint on the size of the weights in the network. We study how well $\cF_{\sigma_k}(M)$ approximates the unit ball of H\"older class $\cH^\alpha$ with smoothness index $\alpha>0$ as $M\to \infty$. Roughly speaking, our main theorem shows that, if $\alpha> (d+2k+1)/2$, then $\cH^\alpha \subseteq \cF_{\sigma_k}(M)$ for some constant $M$ depending on $k,d,\alpha$, and if $\alpha< (d+2k+1)/2$, we obtain the approximation bound
\[
\sup_{h\in \cH^\alpha} \inf_{f\in \cF_{\sigma_k}(M)} \|h-f\|_{L^\infty(\bB^d)} \lesssim M^{-\frac{2\alpha}{d+2k+1-2\alpha}},
\]
where, for two quantities $X$ and $Y$, $X \lesssim Y$ (or $Y \gtrsim X$) denotes the statement that $X\le CY$ for some constant $C>0$ (we will also denote $X \asymp Y$ when $X \lesssim Y \lesssim X$).
In other words, sufficiently smooth functions are always contained in the shallow neural network space $\cF_{\sigma_k}(M)$. And, for less smooth functions, we can characterize the approximation error by the variation norm. Furthermore, combining our result with the random approximation bounds from \cite{bach2017breaking,siegel2022sharp,siegel2023optimal}, we are able to prove that shallow ReLU$^k$ neural network of the form (\ref{shallow NN}) achieves the optimal approximation rate $\cO(N^{-\alpha/d})$ for $\cH^\alpha$ with $\alpha< (d+2k+1)/2$, which generalizes the result of Mhaskar \cite{mhaskar1996neural} to ReLU$^k$ activation. 

In addition to shallow neural networks, we can also apply our results to derive approximation bounds for multi-layer neural networks and convolutional neural networks (CNNs) when $k=1$ (ReLU activation $\sigma:=\sigma_1$). These approximation bounds can then be used to study the performances of machine learning algorithms based on neural networks \cite{anthony2009neural}. Here, we illustrate the idea by studying the nonparametric regression problem. The goal of this problem is to learn a function $h$ from a hypothesis space $\cH$ from its noisy samples
\[
Y_i = h(X_i) + \eta_i, \quad i=1,\dots, n,
\]
where $X_i$ is sampled from an unknown probability distribution $\mu$ and $\eta_i$ is Gaussian noise. One popular algorithm for solving this problem is the empirical least square minimization
\[
\argmin_{f\in \cF_n} \frac{1}{n} \sum_{i=1}^n |f(X_i)- Y_i|^2,
\]
where $\cF_n$ is an appropriately chosen function class. For instance, in deep learning, $\cF_n$ is parameterized by deep neural networks and one solves the minimization by (stochastic) gradient descent methods. Assuming that we can compute a minimizer $f_n^* \in \cF_n$, the performance of the algorithm is often measured by the square loss $\|f_n^*-h\|_{L^2(\mu)}^2$. A fundamental question in learning theory is to determine the convergence rate of the error $\|f_n^*-h\|_{L^2(\mu)}^2 \to 0$ as the sample size $n\to \infty$. The error can be decomposed into two components: approximation error and generalization error (also called estimation error). For neural network models $\cF_n$, our results provide bounds for the approximation errors, while the generalization errors can be bounded by the complexity of the models \cite{shalevshwartz2014understanding,mohri2018foundations}. We study the cases $\cH = \cH^\alpha$ with $\alpha< (d+3)/2$ or $\cH=\cF_{\sigma}(1)$ for three ReLU neural network models: shallow neural network, over-parameterized neural network, and CNN. The models and our contributions are summarized as follows:
\begin{enumerate}[label=\textnormal{(\arabic*)},parsep=0pt]
\item Shallow ReLU neural network $\cF_{\sigma}(N,M)$, where $N$ is the number of neurons and $M$ is a bound for the variation norm that measures the size of the weights. We prove optimal approximation rates (in terms of $N$) for this model. It is also shown that this model can achieve the optimal convergence rates for learning $\cH^\alpha$ and $\cF_{\sigma}(1)$, which complements the recent results for deep neural networks \cite{schmidthieber2020nonparametric,kohler2021rate}.

\item Over-parameterized (deep or shallow) ReLU neural network $\NN(W,L,M)$ studied in \cite{jiao2023approximation}, where $W,L$ are the width and depth respectively, and $M$ is a constraint on the weight matrices. For fixed depth $L$, the generalization error for this model can be controlled by $M$ \cite{jiao2023approximation}. When $\cH=\cH^\alpha$, we characterize the approximation error by $M$, and allow the width $W$ to be arbitrary large so that the model can be over-parameterized (the number of parameters is larger than the number of samples). When $\cH=\cF_{\sigma}(1)$, we can simply increase the width to reduce the approximation error so that the model can also be over-parameterized. Our result shows that this model can achieve nearly optimal convergence rates for learning $\cH^\alpha$ and $\cF_{\sigma}(1)$. Both the approximation and convergence rates improve the results of \cite{jiao2023approximation}.

\item Sparse convolutional ReLU neural network $\CNN(s,L)$ introduced by \cite{zhou2020universality}, where $L$ is the depth and $s\ge 2$ is a fixed integer that controls the filter length. This model is shown to be universal for approximation \cite{zhou2020universality} and universal consistent for regression \cite{lin2022universal}. We improve the approximation bound in \cite{zhou2020universality} and give new convergence rates of this model for learning $\cH^\alpha$ and $\cF_{\sigma}(1)$.
\end{enumerate}
The approximation rates and convergence rates of nonparametric regression for these models are summarized in Table \ref{table}, where we use the notation $a\lor b:= \max\{a,b\}$.

\renewcommand{\arraystretch}{1.4}
\begin{table}[htbp]
\centering
\begin{tabular}{|c|c|c|c|c|}
\hline
\multirow{2}{*}{Model} & \multicolumn{2}{|c|}{Approximation} & \multicolumn{2}{|c|}{Nonparametric regression} \\
\cline{2-5}
&$\cH^\alpha,\alpha<\frac{d+3}{2}$ & $\cF_{\sigma}(1)$ & $\cH^\alpha,\alpha<\frac{d+3}{2}$ & $\cF_{\sigma}(1)$ \\
\hline
$\cF_{\sigma}(N,M)$ & $N^{-\frac{\alpha}{d}} \lor M^{-\frac{2\alpha}{d+3-2\alpha}}$ & $N^{-\frac{1}{2} - \frac{3}{2d}}$ & $n^{-\frac{2\alpha}{d+2\alpha}}$ & $n^{-\frac{d+3}{2d+3}}$ \\ 
\hline
$\NN(W,L,M)$ & $W^{-\frac{\alpha}{d}} \lor M^{-\frac{2\alpha}{d+3-2\alpha}}$ & $W^{-\frac{1}{2} - \frac{3}{2d}}$ & $n^{-\frac{2\alpha}{d+3+2\alpha}}$ & $n^{-\frac{1}{2}}$ \\ 
\hline
$\CNN(s,L)$ & $L^{-\frac{\alpha}{d}}$ & $L^{-\frac{1}{2} - \frac{3}{2d}}$ & $n^{-\frac{\alpha}{d+\alpha}}$ & $n^{-\frac{d+3}{3d+3}}$ \\ 
\hline
\end{tabular}
\caption{Approximation rates and convergence rates of nonparametric regression for three neural network models, ignoring logarithmic factors.}
\label{table}
\end{table}

The rest of the paper is organized as follows. In Section \ref{sec: app}, we present our approximation results for shallow neural networks. Section \ref{sec: proof} gives a proof of our main theorem. In Section \ref{sec: regression}, we apply our approximation results to study these neural network models and derive convergence rates for nonparametric regression using these models. Section \ref{sec: conclusion} concludes this paper with a discussion on possible future directions of research.

\section{Approximation rates for shallow neural networks}\label{sec: app}

Let us begin with some notations for function classes. Let $\bB^d= \{x\in \bR^d:\|x\|_2\le 1\}$ and $\bS^{d-1}=\{x\in\bR^d:\|x\|_2=1\}$ be the unit ball and the unit sphere of $\bR^d$. We are interested in functions of the integral form
\begin{equation}\label{general form}
f(x) = \int_{\bS^d} \sigma_k((x^\intercal,1)v) d\mu(v), \quad x\in \bB^d,
\end{equation}
where $\mu$ is a signed Radon measure on $\bS^d$ with finite total variation $\|\mu\|:=|\mu|(\bS^d)<\infty$ and $\sigma_k(t) := \max\{t,0\}^k$ with $k\in \bN_0:=\bN\cup\{0\}$ is the ReLU$^k$ function (when $k=0$, $\sigma_0(t)$ is the Heaviside function). For simplicity, we will also denote the ReLU function by $\sigma:=\sigma_1$. The variation norm $\gamma(f)$ of $f$ is the infimum of $\|\mu\|$ over all decompositions of $f$ as (\ref{general form}) \cite{bach2017breaking}. By the compactness of $\bS^d$, the infimum can be attained by some signed measure $\mu$. We denote $\cF_{\sigma_k}(M)$ as the function class that contains all functions $f$ in the form (\ref{general form}) whose variation norm $\gamma(f) \le M$, see (\ref{infinite width NN}). The class $\cF_{\sigma_k}(M)$ can be thought of as an infinitely wide neural network with a constraint on its weights. The variation spaces corresponding to shallow neural networks have been studied by many researchers. We refer the reader to \cite{savarese2019how,ongie2020function,bartolucci2023understanding,parhi2022what,parhi2023minimax,siegel2023optimal,siegel2020approximation,siegel2022sharp,siegel2021characterization} for several other definitions and characterizations of these spaces.

We will also need the notion of classical smoothness of functions on Euclidean space. Given a smoothness index $\alpha>0$, we write $\alpha=r+\beta$ where $r\in \bN_0$ and $\beta \in(0,1]$. Let $C^{r,\beta}(\bR^d)$ be the H\"older space with the norm
\[
\|f\|_{C^{r,\beta}(\bR^d)} := \max\left\{ \|f\|_{C^r(\bR^d)}, \max_{\|s\|_1=r}|\partial^s f|_{C^{0,\beta}(\bR^d)} \right\},
\]
where $s=(s_1,\dots,s_d) \in \bN_0^d$ is a multi-index and 
\[
\|f\|_{C^r(\bR^d)} := \max_{\|s\|_1\le r} \|\partial^s f\|_{L^\infty(\bR^d)}, \quad |f|_{C^{0,\beta}(\bR^d)} := \sup_{x\neq y\in \bR^d} \frac{|f(x)-f(y)|}{\|x-y\|_2^\beta}.
\]
Here we use $\|\cdot\|_{L^\infty}$ to denote the supremum norm, since we mainly consider continuous functions. We write $C^{r,\beta}(\bB^d)$ for the Banach space of all
restrictions to $\bB^d$ of functions in $C^{r,\beta}(\bR^d)$. The norm is given by $\|f\|_{C^{r,\beta}(\bB^d)} = \inf\{ \|g\|_{C^{r,\beta}(\bR^d)}: g\in C^{r,\beta}(\bR^d) \mbox{ and } g=f \mbox{ on } \bB^d\}$. For convenience, we will denote the unit ball of $C^{r,\beta}(\bB^d)$ by 
\[
\cH^\alpha:= \left\{ f\in C^{r,\beta}(\bB^d): \|f\|_{C^{r,\beta}(\bB^d)}\le 1 \right\}.
\]
Note that, for $\alpha=1$, $\cH^\alpha$ is a class of Lipschitz continuous functions.

Due to the universality of shallow neural networks \cite{pinkus1999approximation}, $\cF_{\sigma_k}(M)$ can approximate any continuous functions on $\bB^d$ if $M$ is sufficiently large. Our main theorem estimates the rate of this approximation for H\"older class.

\begin{theorem}\label{app norm}
Let $k\in \bN_0$, $d\in \bN$ and $\alpha>0$. If $\alpha> (d+2k+1)/2$ or $\alpha= (d+2k+1)/2$ is an even integer, then $\cH^\alpha \subseteq \cF_{\sigma_k}(M)$ for some constant $M$ depending on $k,d,\alpha$. Otherwise,
\begin{align*}
\sup_{h\in \cH^\alpha} \inf_{f\in \cF_{\sigma_k}(M)} \|h-f\|_{L^\infty(\bB^d)} \lesssim
\begin{cases}
\exp(-\alpha M^2), &\mbox{ if } \alpha= (d+2k+1)/2 \mbox{ and } \alpha/2 \notin \bN, \\
M^{-\frac{2\alpha}{d+2k+1-2\alpha}}, &\mbox{ if } \alpha< (d+2k+1)/2,
\end{cases}
\end{align*}
where the implied constants only depend on $k,d,\alpha$.
\end{theorem}

The proof of Theorem \ref{app norm} is deferred to the next section. Our proof uses similar ideas as \cite[Proposition 3]{bach2017breaking}, which obtained the same approximation rate for $\alpha=1$ (with an additional logarithmic factor). The conclusion is more complicated for the critical value $\alpha=(d+2k+1)/2$. We think this is due to the proof technique and conjecture that $\cH^\alpha \subseteq \cF_{\sigma_k}(M)$ for all $\alpha \ge (d+2k+1)/2$, see Remark \ref{smoothness remark}. Nevertheless, in practical applications of machine learning, the dimension $d$ is large and it is reasonable to expect that $\alpha<(d+2k+1)/2$.

In order to apply Theorem \ref{app norm} to shallow neural networks with finite neurons, we can approximate $\cF_{\sigma_k}(M)$ by the subclass 
\[
\cF_{\sigma_k}(N,M) := \left\{ f(x) = \sum_{i=1}^N a_i\sigma_k((x^\intercal,1)v_i): v_i \in \bS^d, \sum_{i=1}^N |a_i|\le M \right\},
\]
where we restrict the measure $\mu$ to be a discrete one supported on at most $N$ points. The next proposition shows that any function in $\cF_{\sigma_k}(M)$ is the limit of functions in  $\cF_{\sigma_k}(N,M)$ as $N\to \infty$.

\begin{proposition}\label{equivalence}
For $k\in \bN_0$, $\cF_{\sigma_k}(1)$ is the closure of $\cup_{N\in\bN} \cF_{\sigma_k}(N,1)$ in $L^\infty(\bB^d)$.
\end{proposition}
\begin{proof}
Let us denote the closure of $\cup_{N\in\bN} \cF_{\sigma_k}(N,1)$ in $L^\infty(\bB^d)$ by $\widetilde{\cF}_{\sigma_k}(1)$. We first show that $\cF_{\sigma_k}(1) \subseteq \widetilde{\cF}_{\sigma_k}(1)$. For any $f\in \cF_{\sigma_k}(1)$ with the integral form $f(x) = \int_{\bS^d} \sigma_k((x^\intercal,1)v) d\mu(v)$, we can decompose $f$ as
\begin{align*}
f(x) &= \|\mu_+\| \int_{\bS^d} \sigma_k((x^\intercal,1)v) \frac{d\mu_+(v)}{\|\mu_+\|} - \|\mu_-\| \int_{\bS^d} \sigma_k((x^\intercal,1)v) \frac{d\mu_-(v)}{\|\mu_-\|} \\
&=: \|\mu_+\| f_+(x) - \|\mu_-\| f_-(x),
\end{align*}
where $\mu_+$ and $\mu_-$ are the positive and negative parts of $\mu$. If $f_+, f_-\in \widetilde{\cF}_{\sigma_k}(1)$, then $f \in \widetilde{\cF}_{\sigma_k}(1)$. Hence, without loss of generality, we can assume $\mu$ is a probability measure. We are going to approximate $f$ by uniform laws of large numbers. Let $\{v_i\}_{i=1}^N$ be $N$ i.i.d. samples from $\mu$. By symmetrization argument (see \cite[Theorem 4.10]{wainwright2019high} for example), we can bound the expected approximation error by Rademacher complexity \cite{bartlett2002rademacher}:
\[
\bE \left[ \sup_{x\in\bB^d} \left|f(x) - \frac{1}{N}\sum_{i=1}^N \sigma_k((x^\intercal,1)v_i)\right| \right] \le 2 \bE \left[ \sup_{x\in\bB^d} \left| \frac{1}{N} \sum_{i=1}^N \epsilon_i \sigma_k((x^\intercal,1)v_i) \right| \right] =:\cE_k(N),
\]
where $(\epsilon_1,\dots,\epsilon_N)$ is an  i.i.d. sequence of  Rademacher random variables. For $k\in \bN$, the Lipschitz constant of $\sigma_k$ on $[-\sqrt{2},\sqrt{2}]$ is $k2^{(k-1)/2}$. By the contraction property of Rademacher complexity \cite[Corollary 3.17]{ledoux1991probability}, 
\begin{align*}
\cE_k(N) &\le \frac{k2^{(k+1)/2}}{N} \bE \left[ \sup_{x\in\bB^d} \left| \sum_{i=1}^N \epsilon_i (x^\intercal,1)v_i \right| \right] \le \frac{k2^{k/2+1}}{N} \bE \left[ \left\| \sum_{i=1}^N \epsilon_i v_i \right\|_2 \right] \\
&\le \frac{k2^{k/2+1}}{N} \sqrt{\bE \left[ \left\| \sum_{i=1}^N \epsilon_i v_i \right\|_2^2 \right]} = \frac{k2^{k/2+1}}{N} \sqrt{\bE \left[  \sum_{i=1}^N \left\| v_i \right\|_2^2 \right]} = \frac{k2^{k/2+1}}{\sqrt{N}}.
\end{align*}
For $k=0$, the VC dimension of the function class $\{f_x(v)=\sigma_0((x^\intercal,1)v): x\in \bB^d\}$ is at most $d$ \cite[Proposition 4.20]{wainwright2019high}. Thus, we have the bound $\cE_0(N) \lesssim \sqrt{d/N}$ by \cite[Example 5.24]{wainwright2019high}. Hence, $f$ is in the closure of $\cup_{N\in\bN} \cF_{\sigma_k}(N,1)$.

Next, we show that $\widetilde{\cF}_{\sigma_k}(1) \subseteq \cF_{\sigma_k}(1)$ for $k\in \bN_0$. Since $\cup_{N\in\bN} \cF_{\sigma_k}(N,1) \subseteq \cF_{\sigma_k}(1)$, we only need to show that $\cF_{\sigma_k}(1)$ is closed in $L^\infty(\bB^d)$. Let $f_n(x)= \int_{\bS^d} \sigma_k((x^\intercal,1)v) d\mu_n(v)$, where $\|\mu_n\|\le 1$, be a convergent sequence with limit $f \in L^\infty(\bB^d)$. It remains to show that $f\in \cF_{\sigma_k}(1)$.

For $k\in \bN$, by the compactness of $\bS^d$ and Prokhorov's theorem, there exists a weakly convergent subsequence $\mu_{n_i} \to \mu$. In particular, $\|\mu\| \le 1$ and for any $x\in \bB^d$,
\[
\lim_{n_i\to \infty} \int_{\bS^d} \sigma_k((x^\intercal,1)v) d\mu_{n_i}(v) = \int_{\bS^d} \sigma_k((x^\intercal,1)v) d\mu(v) =:\widetilde{f}.
\]
By the compactness of $\bB^d$, $f_{n_i}$ converges uniformly to $\widetilde{f}$. Hence $f= \widetilde{f} \in \cF_{\sigma_k}(1)$.

For $k=0$, we use the idea from \cite[Lemma 3]{siegel2021characterization}. We can view $f_n$ as a Bochner integral $\int_\bD i_{\bD\to L^2(\bB^d)} d\mu_n$ of the inclusion map $i_{\bD\to L^2(\bB^d)}$, where $\bD:=\{g_v(x)=\sigma_0((x^\intercal,1)v): v\in \bS^d \}$. Notice that the set $\bD\subseteq L^2(\bB^d)$ is compact, because the mapping $v\mapsto g_v$ is continuous. By Prokhorov's theorem, there exists a weakly convergent subsequence $\mu_{n_i} \to \mu$. Let us denote $\widetilde{f}= \int_\bD i_{\bD\to L^2(\bB^d)} d\mu$, then $\widetilde{f} \in \cF_{\sigma_k}(1)$ by viewing the Bochner integral as an integral over $\bS^d$. If we choose a countable dense sequence $\{g_j\}_{j=1}^\infty$ of $L^2(\bB^d)$, then the weak convergence implies that 
\[
\lim_{n_i\to \infty} \langle g_j, f_{n_i} \rangle_{L^2(\bB^d)} = \left\langle g_j, \widetilde{f} \right\rangle_{L^2(\bB^d)},
\]
for all $j$. The strong convergence $f_{n_i} \to f$ in $L^\infty(\bB^d)$ implies that the same equality for $f$ replacing $\widetilde{f}$. Therefore, $\langle g_j, f \rangle_{L^2(\bB^d)} = \langle g_j, \widetilde{f} \rangle_{L^2(\bB^d)}$ for all $j$, which shows $f= \widetilde{f} \in \cF_{\sigma_k}(1)$.
\end{proof}

The proof of Proposition \ref{equivalence} actually shows the approximation rate $\cO(N^{-1/2})$ for the subclass $\cF_{\sigma_k}(N,1)$. This rate can be improved if we take into account the smoothness of the activation function. For ReLU activation, Bach \cite[Proposition 1]{bach2017breaking} showed that approximating $f\in \cF_{\sigma}(1)$ by neural networks with finitely many neurons is essentially equivalent to the approximation of a zonoid by zonotopes \cite{bourgain1989approximation,matousek1996improved}. Using this equivalence, he obtained the rate $\cO(N^{-\frac{1}{2}-\frac{3}{2d}})$ for ReLU neural networks. Similar idea was applied to the Heaviside activation in \cite[Theorem 4]{ma2022uniform}, which proved the rate $\cO(N^{-\frac{1}{2}-\frac{1}{2d}})$ for such an activation function. For ReLU$^k$ neural networks, the general approximation rate $\cO(N^{-\frac{1}{2} - \frac{2k+1}{2d}})$ was established in $L^2$ norm by \cite{siegel2022sharp}, which also showed that this rate is sharp. The recent work \cite{siegel2023optimal} further proved that this rate indeed holds in the uniform norm. We summarize their results in the following lemma.

\begin{lemma}[\cite{siegel2023optimal}]\label{app neuron}
For $k\in \bN_0$ and $d\in\bN$, it holds that
\[
\sup_{f \in \cF_{\sigma_k}(1)}\inf_{f_N \in \cF_{\sigma_k}(N,1)} \|f-f_N\|_{L^\infty(\bB^d)} \lesssim N^{-\frac{1}{2} - \frac{2k+1}{2d}}.
\]
\end{lemma}

Combining Theorem \ref{app norm} and Lemma \ref{app neuron}, we can derive the rate of approximation by shallow neural network $\cF_{\sigma_k}(N,M)$ for H\"older class $\cH^\alpha$. Recall that we use the notation $a\lor b:= \max\{a,b\}$.

\begin{corollary}\label{app neuron corollary}
Let $k\in \bN_0$, $d\in \bN$ and $\alpha>0$. 
\begin{enumerate}[label=\textnormal{(\arabic*)},parsep=0pt]
\item If $\alpha> (d+2k+1)/2$ or $\alpha= (d+2k+1)/2$ is an even integer, then there exists a constant $M$ depending on $k,d,\alpha$ such that
\[
\sup_{h\in \cH^\alpha} \inf_{f\in \cF_{\sigma_k}(N,M)} \|h-f\|_{L^\infty(\bB^d)} \lesssim N^{-\frac{1}{2} - \frac{2k+1}{2d}}.
\]

\item If $\alpha= (d+2k+1)/2$ is not an even integer, then there exists $M \asymp \sqrt{\log N}$ such that
\[
\sup_{h\in \cH^\alpha} \inf_{f\in \cF_{\sigma_k}(N,M)} \|h-f\|_{L^\infty(\bB^d)} \lesssim N^{-\frac{1}{2} - \frac{2k+1}{2d}} \sqrt{\log N}.
\]

\item If $\alpha< (d+2k+1)/2$, then 
\[
\sup_{h\in \cH^\alpha} \inf_{f\in \cF_{\sigma_k}(N,M)} \|h-f\|_{L^\infty(\bB^d)} \lesssim N^{-\frac{\alpha}{d}} \lor M^{-\frac{2\alpha}{d+2k+1-2\alpha}}.
\]
Thus, the rate $\cO(N^{-\alpha/d})$ holds when $M\gtrsim N^{(d+2k+1-2\alpha)/(2d)}$.
\end{enumerate}
\end{corollary}
\begin{proof}
We only present the proof for part (3), since other parts can be derived similarly. If $\alpha< (d+2k+1)/2$, then by Theorem \ref{app norm}, for any $h\in \cH^\alpha$, there exists $g\in \cF_{\sigma_k}(K)$ such that $\|h-g\|_{L^\infty(\bB^d)} \lesssim K^{-\frac{2\alpha}{d+2k+1-2\alpha}}$. By Lemma \ref{app neuron}, then there exists $f\in \cF_{\sigma_k}(N,K)$ such that $\|g-f\|_{L^\infty(\bB^d)} \lesssim KN^{-\frac{1}{2} - \frac{2k+1}{2d}}$. If $M \ge  N^{\frac{d+2k+1-2\alpha}{2d}}$, we choose $K=N^{\frac{d+2k+1-2\alpha}{2d}}$ then $f\in \cF_{\sigma_k}(N,K) \subseteq \cF_{\sigma_k}(N,M)$ and
\begin{align*}
\|h-f\|_{L^\infty(\bB^d)} &\le \|h-g\|_{L^\infty(\bB^d)} + \|g-f\|_{L^\infty(\bB^d)} \\
&\lesssim K^{-\frac{2\alpha}{d+2k+1-2\alpha}} + KN^{-\frac{d+2k+1}{2d}} \lesssim N^{-\frac{\alpha}{d}}.
\end{align*}
If $M \le  N^{\frac{d+2k+1-2\alpha}{2d}}$, we choose $K=M$, then
\[
\|h-f\|_{L^\infty(\bB^d)} \lesssim K^{-\frac{2\alpha}{d+2k+1-2\alpha}} + KN^{-\frac{d+2k+1}{2d}} \lesssim M^{-\frac{2\alpha}{d+2k+1-2\alpha}}.
\]
Combining the two bounds gives the desired result.
\end{proof}

We make some comments on the approximation rate for $\cH^\alpha$ with $\alpha< (d+2k+1)/2$. As shown by \cite[Corollary 6.10]{pinkus1999approximation}, the rate $\cO(N^{-\alpha/d})$ in the $L^2$ norm is already known for $\alpha=1,2,\dots,(d+2k+1)/2$. For ReLU activation, the recent paper \cite{mao2023rates} obtained the rate $\cO(N^{-\frac{\alpha}{d} \frac{d+2}{d+4}})$ in the supremum norm. Corollary \ref{app neuron corollary} shows that the rate $\cO(N^{-\alpha/d})$ holds in the supremum norm for all ReLU$^k$ activations. And more importantly, we also provide an explicit control on the network weights to ensure that this rate can be achieved, which is useful for estimating generalization errors (see Section \ref{sec: over-para}). It is well-known that the optimal approximation rate for $\cH^\alpha$ is $\cO(N^{-\alpha/d})$, if we approximate $h\in \cH^\alpha$ by a function class with $N$ parameters and the parameters are continuously dependent on the target function $h$ \cite{devore1989optimal}. However, this result is not directly applicable to neural networks, because we do not have guarantee that the parameters in the network depend continuously on the target function (in fact, this is not true for some constructions \cite{yarotsky2018optimal,lu2021deep,yang2022approximation}). Nevertheless, one can still prove that the rate $\cO(N^{-\alpha/d})$ is optimal for shallow ReLU$^k$ neural networks by arguments based on pseudo-dimension as done in \cite{yarotsky2017error,lu2021deep,yang2022approximation}. 

We describe the idea of proving approximation lower bounds through pseudo-dimension by reviewing the result of Maiorov and Ratsaby \cite{maiorov1999degree} (see also \cite{achour2022general}). Recall that the pseudo-dimension $\Pdim(\cF)$ of a real-valued function class $\cF$ defined on $\bB^d$ is the largest integer $n$ for which there exist points $x_1,\dots,x_n \in \bB^d$ and constants $c_1,\dots,c_n\in \bR$ such that
\begin{equation}\label{pdim}
\left|\{ \sgn(f(x_1)-c_1),\dots,\sgn(f(x_n)-c_n): f\in \cF \}\right| =2^n.
\end{equation}
Maiorov and Ratsaby \cite{maiorov1999degree} introduced a nonlinear $n$-width defined as 
\[
\rho_n(\cH^\alpha) = \inf_{\cF_n} \sup_{h\in \cH^\alpha} \inf_{f\in \cF_n} \|h-f\|_{L^p(\bB^d)},
\]
where $p\in [1,\infty]$ and $\cF_n$ runs over all the classes in $L^p(\bB^d)$ with $\Pdim(\cF_n)\le n$. They constructed a well-separated subclass of $\cH^\alpha$ such that if a function class $\cF$ can approximate this subclass with small error, then $\Pdim(\cF)$ should be large. In other words, the approximation error of any class $\cF_n$ with $\Pdim(\cF_n)\le n$ can be lower bounded. Consequently, they proved that 
\[
\rho_n(\cH^\alpha) \gtrsim n^{-\alpha/d}.
\]
By \cite{bartlett2019nearly}, we can upper bound the pseudo-dimension of shallow ReLU$^k$ neural networks as $n:= \Pdim(\cF_{\sigma_k}(N,M)) \lesssim N \log N$. Hence, 
\[
\sup_{h\in \cH^\alpha} \inf_{f_N\in \cF_{\sigma_k}(N,M)} \|h-f_N\|_{L^p(\bB^d)} \ge \rho_n(\cH^\alpha) \gtrsim (N \log N)^{-\alpha/d},
\]
which shows that the rate $\cO(N^{-\alpha/d})$ in Corollary \ref{app neuron corollary} is optimal in the $L^p$ norm (ignoring logarithmic factors). This also implies the optimality of Theorem \ref{app norm} (otherwise, the proof of Corollary \ref{app neuron corollary} would give a rate better than $\cO(N^{-\alpha/d})$).

\section{Proof of Theorem \ref{app norm}}\label{sec: proof}

Following the idea of \cite{bach2017breaking}, we first transfer the problem to approximation on spheres. Let us begin with a brief review of harmonic analysis on spheres \cite{dai2013approximation}. For $n\in \bN_0$, the spherical harmonic space $\bY_n$ of degree $n$ is the linear space that contains the restrictions of real harmonic homogeneous polynomials of degree $n$ on $\bR^{d+1}$ to the sphere $\bS^d$. The dimension of $\bY_n$ is $N(d,n):= \frac{2n+d-1}{n} \binom{n+d-2}{d-1}$ if $n\neq 0$ and $N(d,n):=1$ if $n=0$.
Spherical harmonics are eigenfunctions of the Laplace-Beltrami operator:
\[
\Delta Y_n = -n(n+d-1)Y_n, \quad Y_n\in \bY_n,
\]
where, in the coordinates $u=(u_1,\dots,u_{d+1}) \in \bS^d$,
\[
\Delta = \sum_{i=1}^d \frac{\partial^2}{\partial u_i^2} - \sum_{i=1}^d \sum_{j=1}^d u_iu_j \frac{\partial^2}{\partial u_i \partial u_j} - d \sum_{i=1}^d u_i \frac{\partial}{\partial u_i}.
\]
Spherical harmonics of different degrees are orthogonal with respect to the inner product $\langle f,g \rangle = \int_{\bS^d} f(u)g(u) d\tau_d(u)$, where $\tau_d$ is the surface area measure of $\bS^d$ (normalized by the surface area $\omega_d:= 2\pi^{(d+1)/2}/\Gamma((d+1)/2)$ so that $\tau_d(\bS^d)=1$). 

Let $\cP_n:L^2(\bS^d)\to \bY_n$ denote the orthogonal projection operator. For any orthonormal basis $\{Y_{nj}:1\le j\le N(d,n)\}$ of $\bY_n$, the addition formula \cite[Theorem 1.2.6]{dai2013approximation} shows
\begin{equation}\label{addition formula}
\sum_{j=1}^{N(d,n)} Y_{nj}(u) Y_{nj}(v) = N(d,n) P_n(u^\intercal v), \quad u,v \in \bS^d,
\end{equation}
where $P_n$ is the Gegenbauer polynomial
\[
P_n(t) := \frac{(-1)^n}{2^n} \frac{\Gamma(d/2)}{\Gamma(n+d/2)} (1-t^2)^{(2-d)/2} \left( \frac{d}{dt}\right)^n (1-t^2)^{n+(d-2)/2}, \quad t\in [-1,1],
\]
with normalization $P_n(1)=1$. Applying the Cauchy–Schwarz inequality to (\ref{addition formula}), we get $|P_n(t)|\le 1$. For $n \neq 0$, $P_n(t)$ is odd (even) if $n$ is odd (even). Note that, for $d=1$ and $n\neq 0$, $N(d,n)=2$ and $P_n(t)$ is the Chebyshev polynomial such that $P_n(\cos \theta)=\cos(n \theta)$. We can write the projection $\cP_n$ as 
\[
\cP_n f(u) = N(d,n) \int_{\bS^d} f(v) P_n(u^\intercal v) d\tau_d(v).
\] 
This motivates the following definition of a convolution operator on the sphere.

\begin{definition}[Convolution]
Let $\varrho$ be the probability distribution with density $c_d (1-t^2)^{(d-2)/2}$ on $[-1,1]$, with the constant $c_d =( \int_{-1}^1 (1-t^2)^{(d-2)/2}dt)^{-1} = \omega_{d-1}/\omega_d$. For $f\in L^1(\bS^d)$ and $g\in L^1_\varrho([-1,1])$, define
\[
(f*g)(u) := \int_{\bS^d} f(v) g(u^\intercal v) d\tau_d(v), \quad u\in \bS^d.
\]
\end{definition}

The convolution on the sphere satisfies Young's inequality \cite[Theorem 2.1.2]{dai2013approximation}: for $p,q,r \ge 1$ with $p^{-1}= q^{-1}+r^{-1}-1$, it holds
\[
\|f*g\|_{L^p(\bS^d)} \le \|f\|_{L^q(\bS^d)} \|g\|_{L^r_\varrho([-1,1])},
\]
where the norm is the uniform one when $r=\infty$. Observe that the projection $\cP_n f = f*(N(d,n) P_n)$ is a convolution operator with $\|N(d,n) P_n\|_{L^\infty([-1,1])} \le N(d,n)$. Furthermore, for $g\in L^1_\varrho([-1,1])$, let $\widehat{g}(n)$ denote the Fourier coefficient of $g$ with respect to the Gegenbauer polynomials,
\[
\widehat{g}(n):= \frac{\omega_{d-1}}{\omega_d} \int_{-1}^1 g(t) P_n(t) (1-t^2)^{(d-2)/2} dt.
\]
By the Funk-Hecke formula, one can show that \cite[Theorem 2.1.3]{dai2013approximation}
\begin{equation}\label{fourier convolution}
\cP_n(f*g) = \widehat{g}(n) \cP_n f, \quad f\in L^1(\bS^d), n\in \bN_0.
\end{equation}
This identity is analogous to the Fourier transform of ordinary convolution. 

One of the key steps in our proof of Theorem \ref{app norm} is the observation that functions of the form $f(u)= \int_{\bS^d} \phi(v) \sigma_k(u^\intercal v) d\tau_d(v)$ are convolutions $\phi *\sigma_k$ with the activation function $\sigma_k\in L^\infty([-1,1])$. \cite[Appendix D.2]{bach2017breaking} has computed the Fourier coefficients $\widehat{\sigma_k}(n)$ explicitly. We summarize the result in the following.

\begin{proposition}\label{fourier sigmak}
For $k\in \bN_0$, $\widehat{\sigma_k}(n)=0$ if and only if $n\ge k+1$ and $n \equiv k\bmod 2$. If $n=0$,
\[
\widehat{\sigma_k}(0) = \frac{\omega_{d-1}}{\omega_d} \frac{\Gamma(d/2)\Gamma((k+1)/2)}{2\Gamma((k+d+1)/2)}.
\]
If $n\ge k+1$ and $n+1 \equiv k\bmod 2$,
\[
\widehat{\sigma_k}(n) = \frac{\omega_{d-1}}{\omega_d} \frac{k!(-1)^{(n-k-1)/2}}{2^n} \frac{\Gamma(d/2)\Gamma(n-k)}{\Gamma((n-k+1)/2) \Gamma((n+d+k+1)/2)}.
\]
By the Stirling formula $\Gamma(x) = \sqrt{2\pi} x^{x-1/2}e^{-x}(1+\cO(x^{-1}))$, we have $\widehat{\sigma_k}(n) \asymp n^{-(d+2k+1)/2}$ for $n\in\bN$ satisfying $\widehat{\sigma_k}(n)\neq 0$.
\end{proposition}

Next, we introduce the smoothness of functions on the sphere. For $0\le \theta \le \pi$, the translation operator $T_\theta$, also called spherical mean operator, is defined by
\[
T_\theta f(u) := \int_{\bS^\perp_u} f(u\cos \theta + v\sin \theta) d\tau_{d-1}(v), \quad u\in \bS^d, f\in L^1(\bS^d),
\]
where $\bS^\perp_u := \{v\in \bS^d:u^\intercal v =0\}$ is the equator in $\bS^d$ with respect to $u$ (hence $\bS^\perp_u$ is isomorphic to the sphere $\bS^{d-1}$). We note that the translation operator satisfies $\cP_n(T_\theta f) = P_n(\cos \theta) \cP_n(f)$. For $\alpha>0$ and $0<\theta<\pi$, we define the $\alpha$-th order difference operator
\[
\Delta_\theta^\alpha := (I-T_\theta)^{\alpha/2} = \sum_{j=0} (-1)^j \binom{\alpha/2}{j} T_\theta^j,
\]
where $\binom{\alpha}{j}=\frac{\alpha(\alpha-1)\cdots(\alpha-j+1)}{j!}$, in a distributional sense by $\cP_n(\Delta_\theta^\alpha f) = (1-P_n(\cos \theta))^{\alpha/2} \cP_n f$, $n\in \bN_0$. For $f\in L^p(\bS^d)$ and $1\le p<\infty$ or $f\in C(\bS^d)$ and $p=\infty$, the $\alpha$-th order modulus of smoothness is defined by
\[
\omega_\alpha(f,t)_p := \sup_{0<\theta\le t} \| \Delta_\theta^\alpha f\|_{L^p(\bS^d)}, \quad 0<t<\pi.
\]
For even integers $\alpha =2s$, one can also use combinations of $T_{j\theta}$ and obtain \cite{rustamov1993equivalence,ditzian2006measures}
\begin{equation}\label{even smoothness equ}
\omega_{2s}(f,t)_p \asymp \sup_{0<\theta\le t} \left\| \sum_{j=0}^{2s} (-1)^j \binom{2s}{j} T_{j\theta} f \right\|_{L^p(\bS^d)}, \quad s\in \bN.
\end{equation}

Another way to characterize the smoothness is through the $K$-functionals. We first introduce the fractional Sobolev space induced by the Laplace-Beltrami operator. We say a function $f\in L^p(\bS^d)$ belong to the Sobolev space $\cW^{\alpha,p}(\bS^d)$ if there exists a function in $L^p(\bS^d)$, which will be denoted by $(-\Delta)^{\alpha/2}f$, such that 
\[
\cP_n((-\Delta)^{\alpha/2}f) = (n(n+d-1))^{\alpha/2} \cP_n f,\quad n\in \bN_0,
\]
where we assume $f,(-\Delta)^{\alpha/2}f\in C(\bS^d)$ for $p=\infty$. 
Then we can define the $\alpha$-th $K$-functional of $f\in L^p(\bS^d)$ as
\[
K_\alpha(f,t)_p := \inf_{g\in \cW^{\alpha,p}(\bS^d)} \left\{ \|f-g\|_{L^p(\bS^d)} + t^\alpha \|(-\Delta)^{\alpha/2} g\|_{L^p(\bS^d)} \right\}, \quad t>0.
\]
It can be shown \cite[Theorem 10.4.1]{dai2013approximation} that the moduli of smoothness and the $K$-functional are equivalent:
\begin{equation}\label{smoothness equ}
\omega_\alpha(f,t)_p 
\asymp K_\alpha(f,t)_p.
\end{equation}

To prove Theorem \ref{app norm}, we denote the function class 
\[
\cG_{\sigma_k}(M) := \left\{g\in L^\infty(\bS^d): g(u) = \int_{\bS^d} \sigma_k(u^\intercal v) d\mu(v), \|\mu\|\le M \right\},
\]
as the corresponding function class of $\cF_{\sigma_k}(M)$ on $\bS^d$. Abusing the notation, we will also denote $\gamma(g) =\inf_{\mu} \|\mu\|$ as the variation norm of $g\in \cG_{\sigma_k}(M)$. The next proposition transfers our approximation problem on the unit ball $\bB^d$ to that on the sphere $\bS^d$.

\begin{proposition}\label{transfer to sphere}
Let $k\in \bN_0$, $d\in \bN$ and $\alpha=r+\beta$ where $r\in \bN_0$ and $\beta \in(0,1]$. Denote $\Omega:= \{(u_1,\dots,u_{d+1})^\intercal\in \bS^d: u_{d+1}\ge 1/\sqrt{2} \}$ and define an operator $S_k:L^\infty(\Omega) \to L^\infty(\bB^d)$ by 
\[
S_kg(x) := (\|x\|_2^2+1)^{k/2} g\left( \frac{1}{\sqrt{\|x\|^2+1}}
\begin{pmatrix}
x \\
1
\end{pmatrix}
\right),\quad x\in \bB^d.
\]
The operator $S_k$ satisfies: (1) If $g\in \cG_{\sigma_k}(M)$, then $S_kg\in \cF_{\sigma_k}(M)$. (2) For any $h\in \cH^\alpha$, there exists $\widetilde{h}\in C(\bS^d)$ such that $S_k \widetilde{h} =h$, $\|\widetilde{h}\|_{L^\infty(\bS^d)}\le C$ and $\omega_{2s^*}(\widetilde{h},t)_\infty \le C t^{\alpha}$, where $s^*\in \bN$ is the smallest integer such that $\alpha\le 2s^*$ and $C$ is a constant independent of $h$. Furthermore, $\widetilde{h}$ can be chosen to be odd or even.
\end{proposition}
\begin{proof}
(1) If $g(u) = \int_{\bS^d} \sigma_k(u^\intercal v) d\mu(v)$ for some $\|\mu\|\le M$, then for $x\in \bB^d$,
\begin{align*}
S_kg(x) &= (\|x\|_2^2+1)^{k/2} \int_{\bS^d} \sigma_k\left((\|x\|_2^2+1)^{-1/2}(x^\intercal,1) v\right) d\mu(v) \\
&= \int_{\bS^d} \sigma_k\left((x^\intercal,1) v\right) d\mu(v).
\end{align*}
Hence, $S_kg \in \cF_{\sigma_k}(M)$ by definition.

(2) Given $h\in \cH^\alpha$, for any $u=(u_1,\dots,u_{d+1})^\intercal \in \Omega$, we define $\widetilde{h}(u) := u_{d+1}^k h(u_{d+1}^{-1}u')$, where $u'=(u_1,\dots,u_d)^\intercal$. It is easy to check that $S_k \widetilde{h} =h$. Note that $h$ is completely determined by the function values of $\widetilde{h}$ on $\Omega$. Observe that the smoothness of $\widetilde{h}$ on $\Omega$ can be controlled by the smoothness of $h$. We can extend $\widetilde{h}$ to $\bR^{d+1}$ so that $\|\widetilde{h}\|_{C^{r,\beta}(\bR^{d+1})} \le C_0$ for some constant $C_0$ independent of $h$, by using (refined version of) Whitney's extension theorem \cite{fefferman2006whitneys,fefferman2009extension,fefferman2020fitting}. It remains to show that $\omega_{2s^*}(\widetilde{h},t)_\infty \lesssim t^{\alpha}$.

By the equivalence (\ref{even smoothness equ}),
\begin{align*}
\omega_{2s^*}(\widetilde{h},t)_\infty &\lesssim \sup_{0<\theta\le t} \sup_{u\in \bS^d} \left| \sum_{j=0}^{2s^*} (-1)^j \binom{2s^*}{j} \int_{\bS^\perp_u} \widetilde{h}(u\cos j\theta + v\sin j\theta) d\tau_{d-1}(v) \right| \\
&\le \sup_{0<\theta\le t} \sup_{u\in \bS^d} \sup_{v\in \bS^\perp_u} \left| \sum_{j=0}^{2s^*} (-1)^j \binom{2s^*}{j} \widetilde{h}(u\cos j\theta + v\sin j\theta) \right| \\
&=: \sup_{0<\theta\le t} \sup_{u\in \bS^d} \sup_{v\in \bS^\perp_u} |H(u,v,\theta)|.
\end{align*}
Next, we estimate $\sup_{0<\theta\le t}|H(u,v,\theta)|$ for small $t>0$ and fixed $u,v$. One can check that the function $f(\cdot):= \widetilde{h}(u\cos (\cdot) + v\sin (\cdot))$ is in $C^{r,\beta}([0,t_0])$ for small $t_0>0$, and $\|f\|_{C^{r,\beta}([0,t_0])} \lesssim \|\widetilde{h}\|_{C^{r,\beta}(\bR^{d+1})}$. Let $\widetilde{\Delta}_\theta f(\cdot) := f(\cdot +\theta) - f(\cdot)$ be the difference operator and $\widetilde{\Delta}_\theta^{n+1}:= \widetilde{\Delta}_\theta \widetilde{\Delta}_\theta^n$ for $n\in \bN$. The binomial theorem shows $H(u,v,\theta) = \widetilde{\Delta}_\theta^{2s^*}f(0)$. Then, the classical theory of moduli of smoothness \cite[Chapter 2.6--2.9]{devore1993constructive} implies 
\[
\sup_{0<\theta\le t}|H(u,v,\theta)| \le \sup_{0<\theta\le t}|\widetilde{\Delta}_\theta^{2s^*}f(0)| \lesssim \|f\|_{C^{r,\beta}([0,t_0])} t^\alpha.
\]
Consequently, we get the desired bound $\omega_{2s^*}(\widetilde{h},t)_\infty \lesssim t^{\alpha}$.

Finally, in order to ensure that $\widetilde{h}$ is odd or even, we can multiply $\widetilde{h}$ by an infinitely differentiable function, which is equal to one on $\Omega$ and zero for $u_{d+1}\le 1/(2\sqrt{2})$, and extend $\widetilde{h}$ to be odd or even. These operations do not decrease the smoothness of $\widetilde{h}$.
\end{proof}

By Proposition \ref{transfer to sphere}, for any $h\in \cH^\alpha$ and $g\in \cG_{\sigma_k}(M)$, we have
\begin{align*}
\|h-S_k g\|_{L^\infty(\bB^d)} = \|S_k \widetilde{h}-S_k g\|_{L^\infty(\bB^d)} \le 2^{k/2} \|\widetilde{h}- g\|_{L^\infty(\bS^d)},
\end{align*}
for some $\widetilde{h}\in C(\bS^d)$. Since $S_kg\in \cF_{\sigma_k}(M)$, we can derive approximation bounds for $\cF_{\sigma_k}(M)$ by studying the approximation capacity of $\cG_{\sigma_k}(M)$. Now, we are ready to prove Theorem \ref{app norm}.

\begin{proof}[Proof of Theorem \ref{app norm}]
By Proposition \ref{transfer to sphere}, for any $h\in \cH^\alpha$, there exists $\widetilde{h}\in C(\bS^d)$ such that $S_k \widetilde{h} =h$, $\|\widetilde{h}\|_{L^\infty(\bS^d)}\le C$ and $\omega_{2s^*}(\widetilde{h},t)_\infty \le Ct^{\alpha}$, where $s^*\in \bN$ is the smallest integer such that $\alpha\le 2s^*$. We choose $\widetilde{h}$ to be odd (even) if $k$ is even (odd). Using $\omega_s(\widetilde{h},t)_2\le 2^{s-2s^*+2}\omega_{2s^*}(\widetilde{h},t)_2$ for $s>2s^*$ \cite[Proposition 10.1.2]{dai2013approximation} and the Marchaud inequality \cite[Eq.(9.6)]{ditzian2006measures}
\[
\omega_s(\widetilde{h},t)_2 \lesssim t^s \left(\int_t^1 \frac{\omega_{2s^*}(\widetilde{h},\theta)_2^2}{\theta^{2s+1}}d\theta \right)^{1/2},\quad s< 2s^*,
\]
we have 
\begin{equation}\label{smoothness estimate}
\omega_s(\widetilde{h},t)_2 \lesssim 
\begin{cases}
t^s, &\mbox{if } s<\alpha, \\
t^\alpha, &\mbox{if } s=\alpha = 2s^*,\\
t^\alpha \sqrt{\log(1/t)}, &\mbox{if } s=\alpha \neq 2s^*,\\
t^\alpha, &\mbox{if } s>\alpha.
\end{cases}
\end{equation}

We study how well $g\in \cG_{\sigma_k}(M)$ approximates $\widetilde{h}$. It turns out that it is enough to consider a subset of $\cG_{\sigma_k}(M)$ that contains functions of the form
\[
g(u) = \int_{\bS^d} \phi(v) \sigma_k(u^\intercal v) d\tau_d(v), \quad u\in \bS^d,
\]
for some $\phi\in L^2(\bS^d)$. Note that $\gamma(g)\le \inf_{\phi} \|\phi\|_{L^1(\bS^d)} \le \inf_{\phi} \|\phi\|_{L^2(\bS^d)}$, where the infimum is taken over all $\phi \in L^2(\bS^d)$ satisfy the integral representation of $g$. Observing that $g=\phi * \sigma_k$ is a convolution, by identity (\ref{fourier convolution}), $\cP_n g = \widehat{\sigma_k}(n) \cP_n\phi$. Hence, we have the Fourier decomposition
\[
g = \sum_{n=0}^\infty \cP_n g =  \sum_{n=0}^\infty \widehat{\sigma_k}(n) \cP_n\phi,
\]
which converges in $L^2(\bS^d)$. This implies that $g\in \cG_{\sigma_k}(M)$ if $g$ is continuous, $\cP_n g=0$ for any  $n\in\bN_0$ satisfying $\widehat{\sigma_k}(n)= 0$ and 
\[
\gamma(g)^2 \le \sum_{\widehat{\sigma_k}(n) \neq 0} \widehat{\sigma_k}(n)^{-2} \|\cP_n g\|^2_{L^2(\bS^d)} \le M^2.
\]
By Proposition \ref{fourier sigmak}, we know that $\widehat{\sigma_k}(n)=0$ if and only if $n\ge k+1$ and $n \equiv k\bmod 2$. For $\widehat{\sigma_k}(n)\neq 0$, we have $\widehat{\sigma_k}(n) \asymp n^{-(d+2k+1)/2}$. 

We consider the convolutions $g_m :=\widetilde{h}*L_m = \int_{\bS^d} \widetilde{h}(u)L_m(u^\intercal v)d\tau_d(v)$ with
\[
L_m(t) := \sum_{n=0}^\infty \eta\left(\frac{n}{m}\right) N(d,n) P_n(t), \quad m\in \bN,
\]
where $\eta$ is a $C^\infty$-function on $[0,\infty)$ such that $\eta(t)=1$ for $0\le t\le 1$ and $\eta(t)=0$ for $t\ge 2$. Since $\eta$ is supported on $[0,2]$, the summation can be terminated at $n=2m-1$, so that $g_m$ is a polynomial of degree at most $2m-1$. Since $\widetilde{h}$ is odd (even) if $k$ is even (odd), $\cP_n g_m= \eta(n/m) \cP_n \widetilde{h}=0$ for any $n \equiv k\bmod 2$. Furthermore, \cite[Theorem 10.3.2]{dai2013approximation} shows that 
\begin{equation}\label{k-functional equ}
K_s(\widetilde{h},m^{-1})_p \asymp \|\widetilde{h} - g_m \|_{L^p(\bS^d)} + m^{-s}\|(-\Delta)^{s/2} g_m\|_{L^p(\bS^d)}.
\end{equation}
By the equivalence (\ref{smoothness equ}) and $\omega_{2s^*}(\widetilde{h},m^{-1})_\infty \lesssim m^{-\alpha}$, the equivalence (\ref{k-functional equ}) for $p=\infty$ implies that we can bound the approximation error as
\[
\|\widetilde{h} - g_m \|_{L^\infty(\bS^d)} \lesssim m^{-\alpha}.
\]
Applying the estimate (\ref{smoothness estimate}) to the equivalence (\ref{k-functional equ}) with $p=2$, we get
\[
\|(-\Delta)^{s/2} g_m\|_{L^2(\bS^d)} \lesssim 
\begin{cases}
1, &\mbox{if } s<\alpha, \\
1, &\mbox{if } s=\alpha = 2s^*,\\
\sqrt{\log m}, &\mbox{if } s=\alpha \mbox{ and } \alpha \neq 2s^*,\\
m^{s-\alpha}, &\mbox{if } \alpha<s \le 2s^*.
\end{cases}
\]
Using $\cP_n((-\Delta)^{s/2}g_m) = (n(n+d-1))^{s/2}\cP_n g_m$, we can estimate the norm $\gamma(g_m)$ as follows
\begin{align*}
\gamma(g_m)^2 &\le \sum_{\widehat{\sigma_k}(n) \neq 0} \widehat{\sigma_k}(n)^{-2} \|\cP_n g_m\|^2_{L^2(\bS^d)} \\
&\lesssim \widehat{\sigma_k}(0)^{-2} \|\cP_0 g_m\|^2_{L^2(\bS^d)} + \sum_{n=1}^{2m-1} n^{d+2k+1} n^{-2s}\|\cP_n((-\Delta)^{s/2}g_m)\|^2_{L^2(\bS^d)} \\
&\lesssim 1+ \sum_{n=1}^{2m-1} n^{d+2k+1-2s}\|\cP_n((-\Delta)^{s/2}g_m)\|^2_{L^2(\bS^d)}\\
&\lesssim 1 + \|(-\Delta)^{s/2}g_m\|^2_{L^2(\bS^d)},
\end{align*}
where we choose $s=(d+2k+1)/2$ in the last inequality. We continue the proof in three different cases.

\textbf{Case I}: $\alpha> (d+2k+1)/2$ or $\alpha= (d+2k+1)/2$ is an even integer. In this case, $s<\alpha$ or $s=\alpha=2s^*$. Thus,
\[
\gamma(g_m)^2 \le \sum_{\widehat{\sigma_k}(n) \neq 0} \widehat{\sigma_k}(n)^{-2} \|\cP_n g_m\|^2_{L^2(\bS^d)} \lesssim 1+ \|(-\Delta)^{s/2}g_m\|^2_{L^2(\bS^d)} \lesssim 1.
\]
Since $\cP_n g_m=\eta(n/m)\cP_n \widetilde{h} = \cP_n \widetilde{h}$ for $n\le m$, we have
\begin{align*}
\gamma(\widetilde{h})^2 &\le \lim_{m\to \infty} \sum_{n\le m,\widehat{\sigma_k}(n) \neq 0} \widehat{\sigma_k}(n)^{-2} \|\cP_n \widetilde{h}\|^2_{L^2(\bS^d)} \\
&\le \lim_{m\to \infty} \sum_{\widehat{\sigma_k}(n) \neq 0} \widehat{\sigma_k}(n)^{-2} \|\cP_n g_m\|^2_{L^2(\bS^d)} \\
&\lesssim 1.
\end{align*}
This shows that $\widetilde{h} \subseteq \cG_{\sigma_k}(M)$ for some constant $M$. Hence, $h=S_k \widetilde{h} \in \cF_{\sigma_k}(M)$ by Proposition \ref{transfer to sphere}.

\textbf{Case II}: $\alpha= (d+2k+1)/2$ is not an even integer. We have $s=\alpha \neq 2s^*$ and 
\[
\gamma(g_m)^2 \lesssim 1+ \|(-\Delta)^{s/2}g_m\|^2_{L^2(\bS^d)} \lesssim \log m.
\]
This shows that $g_m \in \cG_{\sigma_k}(M)$ with $M\lesssim \sqrt{\log m}$. Therefore,
\[
\|\widetilde{h} - g_m\|_{L^\infty(\bS^d)} \lesssim m^{-\alpha} \lesssim \exp(-\alpha M^2).
\]
By Proposition \ref{transfer to sphere}, $f:=S_k g_m \in \cF_{\sigma_k}(M)$ and 
\[
\|h-f\|_{L^\infty(\bB^d)} \le 2^{k/2} \|\widetilde{h}- g_m\|_{L^\infty(\bS^d)} \lesssim \exp(-\alpha M^2).
\]

\textbf{Case III}: $\alpha< (d+2k+1)/2$. In this case, $s>\alpha$ and
\[
\gamma(g_m)^2 \lesssim 1+ \|(-\Delta)^{s/2}g_m\|^2_{L^2(\bS^d)} \lesssim m^{d+2k+1-2\alpha}.
\]
This shows that $g_m \in \cG_{\sigma_k}(M)$ with $M\lesssim m^{(d+2k+1-2\alpha)/2}$. Therefore,
\[
\|\widetilde{h} - g_m\|_{L^\infty(\bS^d)} \lesssim m^{-\alpha} \lesssim M^{-\frac{2\alpha}{d+2k+1-2\alpha}}.
\]
By Proposition \ref{transfer to sphere}, $f:=S_k g_m \in \cF_{\sigma_k}(M)$ and 
\[
\|h-f\|_{L^\infty(\bB^d)} \le 2^{k/2} \|\widetilde{h}- g_m\|_{L^\infty(\bS^d)} \lesssim M^{-\frac{2\alpha}{d+2k+1-2\alpha}},
\]
which finishes the proof.
\end{proof}

\begin{remark}\label{smoothness remark}
Since we are only able to estimate the smoothness $\omega_{2s^*}(\widetilde{h},t)_\infty$ for even integer $2s^*$, we have an extra logarithmic factor for the bound $\omega_s(\widetilde{h},t)_2 \lesssim t^\alpha \sqrt{\log(1/t)}$ in (\ref{smoothness estimate}) when $s=\alpha \neq 2s^*$, due to the Marchaud inequality. Consequently, we can only obtain exponential convergence rate when $\alpha= (d+2k+1)/2$ is not an even integer. We conjecture the bound $\omega_s(\widetilde{h},t)_2 \lesssim t^\alpha$ holds for all $s\ge\alpha$. If this is the case, then the proof of Theorem \ref{app norm} implies $\cH^\alpha \subseteq \cF_{\sigma_k}(M)$ for some constant $M$ when $\alpha \ge (d+2k+1)/2$.
\end{remark}

\section{Nonparametric regression}\label{sec: regression}

In this section, we apply our approximation results to nonparametric regression using neural networks. For simplicity, we will only consider ReLU activation function ($k=1$), which is the most popular activation in deep learning. 

We study the classical problem of learning a $d$-variate function $h\in \cH$ from its noisy samples, where we will assume $\cH = \cH^\alpha$ with $\alpha<(d+3)/2$ or $\cH = \cF_{\sigma}(1)$. Note that, due to Theorem \ref{app norm}, the results for $\cF_{\sigma}(1)$ can be applied to $\cH^\alpha$ with $\alpha>(d+3)/2$ by scaling the variation norm. Suppose we have a data set of $n\ge 2$ samples $\cD_n = \{(X_i,Y_i)\}_{i=1}^n \subseteq \bB^d \times \bR$ which are independently and identically generated from the regression model
\begin{equation}\label{regression model}
Y_i = h(X_i) + \eta_i, \quad X_i \sim \mu, \quad \eta_i \sim \cN(0,V^2), \quad i=1,\dots,n, \quad h\in \cH,
\end{equation}
where $\mu$ is the marginal distribution of the covariates $X_i$ supported on $\bB^d$, and $\eta_i$ is an i.i.d. Gaussian noise independent of $X_i$ (we will treat the variance $V^2$ as a fixed constant). We are interested in the empirical risk minimizer (ERM)
\begin{equation}\label{ERM}
f_n^* \in \argmin_{f\in \cF_n} \cL_n(f) := \argmin_{f\in \cF_n} \frac{1}{n} \sum_{i=1}^n |f(X_i)- Y_i|^2,
\end{equation}
where $\cF_n$ is a function class parameterized by neural networks. For simplicity, we assume here and in the sequel that the minimum above indeed exists. The performance of the estimation is measured by the expected risk
\[
\cL(f) := \bE_{(X,Y)} [(f(X)-Y)^2] = \bE_{X\sim \mu} [(f(X) - h(X))^2]  + V^2.
\]
It is equivalent to evaluating the estimator by the excess risk
\[
\| f - h\|_{L^2(\mu)}^2 = \cL(f) - \cL(h).
\]

In the statistical analysis of learning algorithms, we often require that the hypothesis class is uniformly bounded. We define the truncation operator $\cT_B$ with level $B>0$ for real-valued functions $f$ as
\[
\cT_Bf(x) := 
\begin{cases}
f(x) &\quad \mbox{if }|f(x)|\le B, \\
\sgn(f(x)) B &\quad \mbox{if } |f(x)|> B.
\end{cases}
\]
Since we always assume the regression function $h$ is bounded, truncating the output of the estimator $f_n^*$ appropriately dose not increase the excess risk. We will estimate the convergence rate of $\bE_{\cD_n} \|\cT_{B_n}f_n^*-h\|_{L^2(\mu)}^2$, where $B_n \lesssim \log n$, as the number of samples $n\to \infty$.

\subsection{Shallow neural networks}

The rate of convergence of neural network regression estimates has been analyzed by many papers \cite{mccaffrey1994convergence,kohler2005adaptive,bauer2019deep,schmidthieber2020nonparametric,nakada2020adaptive,kohler2021rate}. It is well-known that the optimal minimax rate of convergence for learning a regression function $h\in \cH^\alpha$ is $n^{-2\alpha/(d+2\alpha)}$ \cite{stone1982optimal}. This optimal rate has been established (up to logarithmic factors) for two-hidden-layers neural networks with certain squashing activation functions \cite{kohler2005adaptive} and for deep ReLU neural networks \cite{schmidthieber2020nonparametric,kohler2021rate}. For shallow networks, \cite{mccaffrey1994convergence} proved a rate of $n^{-2\alpha/(2\alpha+d+5)+\epsilon}$ with $\epsilon>0$ for a certain cosine squasher activation function. However, to the best of our knowledge, it is unknown whether shallow neural networks can achieve the optimal rate. In this section, we provide an affirmative answer to this question by proving that shallow ReLU neural networks can achieve the optimal rate for $\cH^\alpha$ with $\alpha<(d+3)/2$.

We will use the following lemma to analyze the convergence rate. It decomposes the error of the ERM into generalization error and approximation error, and bounds the generalization error by the covering number of the hypothesis class $\cF_n$.

\begin{lemma}[\cite{kohler2021rate}]\label{gen bound by covering}
Let $f_n^*$ be the estimator (\ref{ERM}) and set $B_n = c_1\log n$ for some constant $c_1>0$. Then, 
\begin{align*}
&\bE_{\cD_n} \|\cT_{B_n}f_n^*-h\|_{L^2(\mu)}^2 \\
\le & \frac{c_2 (\log n)^2 \sup_{X_{1:n}\in (\bB^d)^n}\log (\cN(n^{-1}B_n^{-1}, \cT_{B_n}\cF_n,\|\cdot\|_{L^1(X_{1:n})})+1)}{n} + 2 \inf_{f\in \cF_n} \|f-h\|_{L^2(\mu)}^2,
\end{align*}
for $n>1$ and some constant $c_2>0$ (independent of $n$ and $f_n^*$), where $X_{1:n} =(X_1,\dots,X_n)$ denotes a sequence of sample points in $\bB^d$ and $\cN(\epsilon, \cT_{B_n}\cF_n,\|\cdot\|_{L^1(X_{1:n})})$ denotes the $\epsilon$-covering number of the function class $\cT_{B_n}\cF_n:=\{\cT_{B_n}f,f\in \cF_n\}$ in the metric $\|f-g\|_{L^1(X_{1:n})} = \frac{1}{n}\sum_{i=1}^n|f(X_i)-g(X_i)|$.
\end{lemma}

For shallow neural network model $\cF_n= \cF_{\sigma}(N_n,M_n)$, Lemma \ref{app neuron} and Corollary \ref{app neuron corollary} provide bounds for the approximation errors. The covering number of the function class $\cT_{B_n}\cF_n$ can be estimated by using the pseudo-dimension of $\cT_{B_n}\cF_n$ \cite{haussler1992decision}. Choosing $N_n,M_n$ appropriately to balance the approximation and generalization errors, we can derive convergence rates for the ERM.

\begin{theorem}\label{regression shallow}
Let $f_n^*$ be the estimator (\ref{ERM}) with $\cF_n = \cF_{\sigma}(N_n,M_n)$ and set $B_n = c_1\log n$ for some constant $c_1>0$. 
\begin{enumerate}[label=\textnormal{(\arabic*)},parsep=0pt]
\item If $\cH = \cH^\alpha$ with $\alpha<(d+3)/2$, we choose
\[
N_n \asymp n^{\frac{d}{d+2\alpha}},  \quad M_n \gtrsim n^{\frac{d+3-2\alpha}{2d+4\alpha}},
\]
then
\[
\bE_{\cD_n} \|\cT_{B_n}f_n^*-h\|_{L^2(\mu)}^2 \lesssim n^{-\frac{2\alpha}{d+2\alpha}} (\log n)^4.
\]

\item If $\cH = \cF_{\sigma}(1)$, we choose
\[
N_n \asymp n^{\frac{d}{2d+3}},  \quad M_n \ge 1,
\]
then
\[
\bE_{\cD_n} \|\cT_{B_n}f_n^*-h\|_{L^2(\mu)}^2 \lesssim  n^{-\frac{d+3}{2d+3}} (\log n)^4.
\]
\end{enumerate}
\end{theorem}
\begin{proof}
To apply the bound in Lemma \ref{gen bound by covering}, we need to estimate the covering number $\cN(\epsilon, \cT_{B_n}\cF_n,\|\cdot\|_{L^1(X_{1:n})})$. The classical result of \cite[Theorem 6]{haussler1992decision} showed that the covering number can be bounded by pseudo-dimension:
\begin{equation}\label{cover bound}
\log \cN(\epsilon, \cT_{B_n}\cF_n,\|\cdot\|_{L^1(X_{1:n})}) \lesssim \Pdim(\cT_{B_n}\cF_n) \log(B_n/\epsilon),
\end{equation}
where $\Pdim(\cT_{B_n}\cF_n)$ is the pseudo-dimension of the function class $\cT_{B_n}\cF_n$, see (\ref{pdim}). For ReLU neural networks, \cite{bartlett2019nearly} showed that
\[
\Pdim(\cT_{B_n}\cF_n) \lesssim N_n \log N_n.
\]
Consequently, we have
\[
\log \cN(\epsilon, \cT_{B_n}\cF_n,\|\cdot\|_{L^1(X_{1:n})}) \lesssim N_n \log (N_n) \log(B_n/\epsilon).
\]

Applying Lemma \ref{gen bound by covering} and Corollary \ref{app neuron corollary}, if $\cH= \cH^\alpha$ with $\alpha<(d+3)/2$, then 
\[
\bE_{\cD_n} \|\cT_{B_n}f_n^*-h\|_{L^2(\mu)}^2 \lesssim \frac{(\log n)^2 N_n \log (N_n) \log(n B_n^2)}{n} + N_n^{-\frac{2\alpha}{d}} \lor M_n^{-\frac{4\alpha}{d+3-2\alpha}}.
\]
By choosing $N_n \asymp n^{d/(d+2\alpha)}$ and $M_n \gtrsim N_n^{(d+3-2\alpha)/(2d)}$, we get $\bE_{\cD_n} \|\cT_{B_n}f_n^*-h\|_{L^2(\mu)}^2 \lesssim n^{-2\alpha/(d+2\alpha)} (\log n)^4$. Similarly, by Lemma \ref{gen bound by covering} and Lemma \ref{app neuron}, if $\cH= \cF_{\sigma}(1)$ and $M_n \ge 1$, then
\[
\bE_{\cD_n} \|\cT_{B_n}f_n^*-h\|_{L^2(\mu)}^2 \lesssim \frac{(\log n)^2 N_n \log (N_n) \log(n B_n^2)}{n} + N_n^{-\frac{d+3}{d}}.
\]
We choose $N_n \asymp n^{d/(2d+3)}$, then $\bE_{\cD_n} \|\cT_{B_n}f_n^*-h\|_{L^2(\mu)}^2 \lesssim n^{-(d+3)/(2d+3)} (\log n)^4$.
\end{proof}

\begin{remark}
The Gaussian noise assumption on the model (\ref{regression model}) can be weaken for Lemma \ref{gen bound by covering} and hence for Theorem \ref{regression shallow}. We refer the reader to \cite[Appendix B, Lemma 18]{kohler2021rate} for more details. Theorem \ref{regression shallow} can be easily generalized to shallow ReLU$^k$ neural networks for $k\ge 1$ by using the same proof technique. For example, one can show that, if $h\in \cH^\alpha$ with $\alpha<(d+2k+1)/2$, then we can choose $\cF_n = \cF_{\sigma_k}(N_n,M_n)$ with $N_n \asymp n^{\frac{d}{d+2\alpha}}$ and $M_n \gtrsim n^{\frac{d+2k+1-2\alpha}{2d+4\alpha}}$, such that $\bE_{\cD_n} \|\cT_{B_n}f_n^*-h\|_{L^2(\mu)}^2 \lesssim n^{-\frac{2\alpha}{d+2\alpha}} (\log n)^4$. 
\end{remark}

Theorem \ref{regression shallow} shows that least square minimization using shallow ReLU neural networks can achieve the optimal rate $n^{-\frac{2\alpha}{d+2\alpha}}$ for learning functions in $\cH^\alpha$ with $\alpha<(d+3)/2$. For the function class $\cF_{\sigma}(1)$, the rate $n^{-\frac{d+3}{2d+3}}$ is also minimax optimal as proven by \cite[Lemma 25]{parhi2023minimax} (they studied a slightly different function class, but their result also holds for $\cF_{\sigma}(1)$). Specifically, \cite[Theorem 4 and Theorem 8]{siegel2022sharp} give a sharp estimate for the metric entropy 
\[
\log \cN(\epsilon, \cF_{\sigma}(1), \|\cdot\|_{L^2(\bB^d)}) \asymp \epsilon^{-\frac{2d}{d+3}}.
\]
Combining this estimate with the classical result of Yang and Barron (see \cite[Proposition 1]{yang1999information} and \cite[Chapter 15]{wainwright2019high}), we get
\[
\inf_{\widehat{f}_n} \sup_{h\in \cF_{\sigma}(1)} \bE_{\cD_n} \|\widehat{f}_n-h\|_{L^2(\bB^d)}^2 \asymp n^{-\frac{d+3}{2d+3}},
\]
where the infimum taken is over all estimators based on the samples $\cD_n$, which are generated from the model (\ref{regression model}).

\subsection{Deep neural networks and over-parameterization}\label{sec: over-para}

There is a direct way to generalize the analysis in the last section to deep neural networks: we can implement shallow neural networks by sparse multi-layer neural networks with the same order of parameters, and estimate the approximation and generalization performance of the constructed networks. Since the optimal convergence rates of deep neural networks have already been established in \cite{schmidthieber2020nonparametric,kohler2021rate}, we do not pursue in this direction. Instead, we study the convergence rates of over-parameterized neural networks by using the idea discussed in \cite{jiao2023approximation}. The reason for studying such networks is that, in modern applications of deep learning, the number of parameters in the networks is often much larger than the number of samples. However, in the convergence analysis of \cite{schmidthieber2020nonparametric,kohler2021rate}, the network that achieves the optimal rate is under-parameterized (see also the choice of $N_n$ in Theorem \ref{regression shallow}). Hence, the analysis can not explain the empirical performance of deep learning models used in practice.

Following \cite{jiao2023approximation}, we consider deep neural networks with norm constraints on weight matrices. For $W,L\in \bN$, we denote by $\NN(W,L)$ the set of functions
that can be parameterized by ReLU neural networks in the form
\begin{equation}\label{deep NN}
\begin{aligned}
f^{(0)}(x) &= x \in \bR^d, \\
f^{(\ell+1)}(x) &= \sigma(A^{(\ell)} f^{(\ell)}(x)+b^{(\ell)}), \quad \ell = 0,\dots,L-1, \\
f(x) &= A^{(L)} f^{(L)}(x) + b^{(L)},
\end{aligned}
\end{equation}
where $A^{(\ell)} \in \bR^{N_{\ell+1}\times N_{\ell}}$, $b^{(\ell)}\in \bR^{N_{\ell+1}}$ with $N_0 =d,N_{L+1} =1$ and $\max\{N_1,\dots,N_L\} =W$. The numbers $W$ and $L$ are called the width and depth of the neural network, respectively. Let us use the notation $f_\theta$ to emphasize that the neural network function is parameterized by $\theta =((A^{(0)},b^{(0)}),\dots,(A^{(L)},b^{(L)}))$. We can define a norm constraint on the weight matrices as follows
\[
\kappa(\theta) :=  \|(A^{(L)},b^{(L)})\| \prod_{\ell =0}^{L-1} \max\left\{\| (A^{(\ell)},b^{(\ell)})\|,1\right\},
\]
where we use $\| A\| := \sup_{\|x\|_\infty \le 1} \|Ax\|_\infty$ to denote the operator norm (induced by the $\ell^\infty$ norm) of a matrix $A = (a_{i,j}) \in \bR^{m\times n}$. It is well-known that $\| A\|$ is the maximum $1$-norm of the rows of $A$:
\[
\| A\| = \max_{1\le i\le m} \sum_{j=1}^{n} |a_{i,j}|.
\]
The motivation for such a definition of $\kappa(\theta)$ is discussed in \cite{jiao2023approximation}. For $M\ge 0$, we denote by $\NN(W,L,M)$ as the set of functions $f_\theta \in \NN(W,L)$ that satisfy $\kappa(\theta) \le M$. It is shown by \cite[Proposition 2.5]{jiao2023approximation} that, if $W_1\le W_2,L_1\le L_2, M_1\le M_2$, then $\NN(W_1,L_1,M_1) \subseteq \NN(W_2,L_2,M_2)$. (Strictly speaking, \cite{jiao2023approximation} use the convention that the bias $b^{(L)}=0$ in the last layer. But the results can be easily generalized to the case $b^{(L)}\neq 0$, see \cite[Section 2.1]{yang2022learning} for details.) 

To derive approximation bounds for deep neural networks, we consider the relationship of $\cF_{\sigma}(N,M)$ and $\NN(N,1,M)$. The next proposition shows the function classes $\NN(N,1,M)$ and $\cF_{\sigma}(N,M)$ have essentially the same approximation power.

\begin{proposition}
For any $N\in \bN$ and $M>0$, we have $\cF_{\sigma}(N,M)\subseteq \NN(N,1,\sqrt{d+1}M)$ and $\NN(N,1,M) \subseteq \cF_{\sigma}(N+1,M)$.
\end{proposition}
\begin{proof}
Each function $f(x) =\sum_{i=1}^N a_i\sigma((x^\intercal,1)v_i)$ in $\cF_{\sigma}(N,M)$ can be parameterized in the form (\ref{deep NN}) with $W=N,L=1$ and
\[
(A^{(0)}, b^{(0)}) = (v_1,\dots,v_N)^\intercal,
\quad (A^{(1)}, b^{(1)}) = (a_1,\dots,a_N,0).
\]
Since $v_i\in \bS^d$, it is easy to see that $\kappa(\theta) \le \sqrt{d+1} M$. Hence, $\cF_{\sigma}(N,M)\subseteq \NN(N,1,\sqrt{d+1}M)$. 

On the other side, let $f_\theta\in \NN(N,1,M)$ be a function parameterized in the form (\ref{deep NN}) with $(A^{(0)}, b^{(0)}) = (a^{(0)}_1,\dots,a^{(0)}_N)^\intercal$ and $(A^{(1)}, b^{(1)}) = (a^{(1)}_1,\dots,a^{(1)}_N,b^{(1)})$, where $a^{(0)}_i\in \bR^{d+1}$ and $a^{(1)}_i, b^{(1)}\in \bR$. Then, $f_\theta$ can be represented as
\[
f_\theta(x) = \sum_{i=1}^N a^{(1)}_i \|a^{(0)}_i\|_2 \sigma \left((x^\intercal,1) \frac{a^{(0)}_i}{\|a^{(0)}_i\|_2} \right) + b^{(1)}\sigma(1),
\]
where we assume $\|a^{(0)}_i\|_2 \neq 0$ without loss of generality. Since
\[
\gamma(f_\theta) \le \sum_{i=1}^N |a^{(1)}_i| \|a^{(0)}_i\|_2 + |b^{(1)}| \le \|(A^{(0)}, b^{(0)})\| \sum_{i=1}^N |a^{(1)}_i| + |b^{(1)}| \le \kappa(\theta),
\]
we conclude that $\NN(N,1,M) \subseteq \cF_{\sigma}(N+1,M)$. 
\end{proof}

As a corollary of Theorem \ref{app norm} and Lemma \ref{app neuron}, we get the following approximation bounds for deep neural networks.

\begin{corollary}\label{app overp}
For $\cH^\alpha$ with $0<\alpha<(d+3)/2$, we have
\[
\sup_{h\in \cH^\alpha} \inf_{f\in \NN(W,L,M)} \|h-f\|_{L^\infty(\bB^d)} \lesssim W^{-\frac{\alpha}{d}} \lor M^{-\frac{2\alpha}{d+3-2\alpha}}.
\]
For $\cF_{\sigma}(1)$, there exists a constant $M\ge 1$ such that
\[
\sup_{h \in \cF_{\sigma}(1)} \inf_{f\in \NN(W,L,M)} \|h-f\|_{L^\infty(\bB^d)} \lesssim W^{-\frac{1}{2} - \frac{3}{2d}}.
\]
\end{corollary}
\begin{proof}
The first part is a direct consequence of Corollary \ref{app neuron corollary} and the inclusion $\cF_{\sigma}(W,M) \subseteq \NN(W,L,\sqrt{d+1}M)$. The second part follows from Lemma \ref{app neuron} and we can choose $M=\sqrt{d+1}$.
\end{proof}

In the first part of Corollary \ref{app overp}, if we allow the width $W$ to be arbitrary large, say $W\gtrsim M^{2d/(d+3-2\alpha)}$, then we can bound the approximation error by the size of weights. Hence, this result can be applied to over-parameterized neural networks. (Note that, in Theorem \ref{regression shallow}, we use a different regime of the bound.) For the approximation of $\cF_{\sigma}(1)$, the size of weights is bounded by a constant. We will show that this constant can be used to control the generalization error. Since the approximation error is bounded by $W$ and is independent of $M$, we do not have trade-off in the error decomposition of ERM and only need to choose $W$ sufficiently large to reduce the approximation error. Hence, it can also be applied to over-parameterized neural networks.

The approximation rate $M^{-\frac{2\alpha}{d+3-2\alpha}}$ for $\cH^\alpha$ in Corollary \ref{app overp} improves the rate $M^{-\frac{\alpha}{d+1}}$ proven by \cite{jiao2023approximation}. Using the upper bound for Rademacher complexity of $\NN(W,L,M)$ (see Lemma \ref{Rademacher bound}), \cite{jiao2023approximation} also gave an approximation lower bound $(M\sqrt{L})^{-\frac{2\alpha}{d-2\alpha}}$. For fixed depth $L$, our upper bound is very close to this lower bound. We conjecture that the rate in Corollary \ref{app overp} is optimal with respect to $M$ (for fixed depth $L$). The discussion of optimality at the end of Section \ref{sec: app} implies that the conjecture is true for shallow neural networks (i.e. $L=1$).

To control the generalization performance of over-parameterized neural networks, we need to have size-independent sample complexity bounds for such networks. Several methods have been applied to obtain such kind of bounds in recent works \cite{neyshabur2015norm,neyshabur2018pac,bartlett2017spectrally,golowich2020size}. Here, we will use the result of \cite{golowich2020size}, which estimates the Rademacher complexity of deep neural networks \cite{bartlett2002rademacher}. For a set $S\subseteq \bR^n$, let us denote its Rademacher complexity by
\[
\cR_n(S) := \bE_{\xi_{1:n}} \left[ \sup_{(s_1,\dots,s_n)\in S} \frac{1}{n} \sum_{i=1}^n \xi_i s_i  \right],
\]
where $\xi_{1:n} = (\xi_1,\dots, \xi_n)$ is a sequence of i.i.d. Rademacher random variables. The following lemma is from \cite[Theorem 3.2]{golowich2020size} and \cite[Lemma 2.3]{jiao2023approximation}.

\begin{lemma}\label{Rademacher bound}
For any $x_1,\dots,x_n \in [-1,1]^d$, let $S:= \{(f(x_1),\dots,f(x_n)) :f \in \NN(W,L,M) \} \subseteq \bR^n$, then
\[
\cR_n(S) \le \frac{M\sqrt{2(L+2+\log(d+1))}}{\sqrt{n}}.
\]
\end{lemma}

Now, we can estimate the convergence rates of the ERM based on over-parameterized neural networks. As usual, we decompose the excess risk of the ERM into approximation error and generalization error, and bound them by Corollary \ref{app overp} and Lemma \ref{Rademacher bound}, respectively. Note that the convergence rates in the following theorem are worse than the optimal rates in Theorem \ref{regression shallow}.

\begin{theorem}
Let $f_n^*$ be the estimator (\ref{ERM}) with $\cF_n = \{ \cT_{B_n}f: f\in\NN(W_n,L,M_n) \}$, where $L\in \bN$ is a fixed constant, $1\le B_n \lesssim \log n$ in case (1) and $\sqrt{2} \le B_n \lesssim \log n$ in case (2). 
\begin{enumerate}[label=\textnormal{(\arabic*)},parsep=0pt]
\item If $\cH = \cH^\alpha$ with $\alpha<(d+3)/2$, we choose
\[
W_n \gtrsim n^{\frac{d}{d+3+2\alpha}}, \quad M_n \asymp n^{\frac{1}{2} - \frac{2\alpha}{d+3+2\alpha}},
\]
then
\[
\bE_{\cD_n} \|f_n^*-h\|_{L^2(\mu)}^2 \lesssim n^{-\frac{2\alpha}{d+3+2\alpha}} \log n.
\]

\item If $\cH = \cF_{\sigma}(1)$, we choose a large enough constant $M$ and let
\[
W_n \gtrsim n^{\frac{d}{2d+6}}, \quad M_n=M,
\]
then
\[
\bE_{\cD_n} \|f_n^*-h\|_{L^2(\mu)}^2 \lesssim n^{-\frac{1}{2}} \log n.
\]
\end{enumerate}
\end{theorem}
\begin{proof}
The proof is essentially the same as \cite[Theorem 4.1]{jiao2023approximation}. Observe that, for any $f\in \cF_n$, 
\begin{align*}
&\| f_n^* - h\|^2_{L^2(\mu)} = \cL(f_n^*) - \cL(h) \\
=& \left[\cL(f_n^*) - \cL_n(f_n^*) \right] + \left[\cL_n(f_n^*) - \cL_n(f)\right] + \left[\cL_n(f) - \cL(f)\right] + \left[\cL(f) - \cL(h)\right] \\
\le& \left[\cL(f_n^*) - \cL_n(f_n^*) \right] + \left[\cL_n(f) - \cL(f)\right] + \| f - h\|^2_{L^2(\mu)}.
\end{align*}
Using $\bE_{\cD_n} [\cL_n(f)] = \cL(f)$ and taking the infimum over $f\in \cF_n$, we get
\begin{equation}\label{temp1}
\bE_{\cD_n} \|f_n^*-h\|_{L^2(\mu)}^2 \le \inf_{f\in \cF_n} \| f - h\|^2_{L^2(\mu)} + \bE_{\cD_n} \left[\cL(f_n^*) - \cL_n(f_n^*) \right].
\end{equation}

Let us denote the collections of sample points and noises by $X_{1:n} =(X_1,\dots,X_n)$ and $\eta_{1:n}=(\eta_1,\dots,\eta_n)$. We can bound the generalization error as follows
\begin{align}
&\bE_{\cD_n} \left[\cL(f_n^*) - \cL_n(f_n^*) \right] \notag\\
=& \bE_{\cD_n} \left[ \| f_n^* - h\|^2_{L^2(\mu)} + V^2 - \left(\frac{1}{n} \sum_{i=1}^n (f_n^*(X_i)- h(X_i))^2 -2 \eta_i(f_n^*(X_i)- h(X_i)) + \eta_i^2\right) \right] \notag\\
=& \bE_{\cD_n}  \left[ \| f_n^* - h\|^2_{L^2(\mu)}  - \frac{1}{n} \sum_{i=1}^n (f_n^*(X_i)- h(X_i))^2 \right] +2 \bE_{\cD_n}\left[ \frac{1}{n} \sum_{i=1}^n \eta_i(f_n^*(X_i)- h(X_i)) \right] \notag\\
\le& \bE_{X_{1:n}} \left[ \sup_{\phi\in \Phi_n} \left(\bE_X[\phi^2(X)] - \frac{1}{n} \sum_{i=1}^n \phi^2(X_i) \right) \right] + 2 \bE_{X_{1:n}} \bE_{\eta_{1:n}} \left[ \sup_{\phi\in \Phi_n} \frac{1}{n} \sum_{i=1}^n \eta_i \phi(X_i) \right], \label{temp2}
\end{align}
where we denote $\Phi_n:= \{f-h:f\in \cF_n\}$. By a standard symmetrization argument (see \cite[Theorem 4.10]{wainwright2019high}), we can bound the first term in (\ref{temp2}) by the Rademacher complexity:
\[
\bE_{X_{1:n}} \left[ \sup_{\phi\in \Phi_n} \left(\bE_X[\phi^2(X)] - \frac{1}{n} \sum_{i=1}^n \phi^2(X_i) \right) \right] \le 2 \bE_{X_{1:n}} \left[ \cR_n(\Phi_n^2(X_{1:n})) \right],
\]
where $\Phi_n^2(X_{1:n}):= \{ (\phi^2(X_1),\dots,\phi^2(X_n)) \in \bR^n: \phi\in \Phi_n \} \subseteq \bR^n$ is the set of function values on the sample points. Recall that we assume $B_n\ge 1$ in case (1) and $B_n\ge \sqrt{2}$ in case (2). Hence, we always have $\|h\|_{L^\infty(\bB^d)} \le \sqrt{2}$ and $\|\phi\|_{L^\infty(\bB^d)} \le 2B_n$ for any $\phi\in \Phi_n$. By the structural properties of Rademacher complexity \cite[Theorem 12]{bartlett2002rademacher}, 
\begin{align*}
\bE_{X_{1:n}} \left[ \cR_n(\Phi_n^2(X_{1:n})) \right] &\le 8 B_n \bE_{X_{1:n}} \left[ \cR_n(\Phi_n(X_{1:n})) \right] \\
&\le 8B_n \left( \bE_{X_{1:n}} \left[ \cR_n(\cF_n(X_{1:n})) \right] + \frac{\|h\|_{L^\infty(\bB^d)}}{\sqrt{n}} \right) \\
&\lesssim \frac{M_n \log n}{\sqrt{n}},
\end{align*}
where we apply Lemma \ref{Rademacher bound} in the last inequality. Note that the second term in (\ref{temp2}) is a Gaussian complexity. We can also bound it by the Rademacher complexity \cite[Lemma 4]{bartlett2002rademacher}:
\[
\bE_{X_{1:n}} \bE_{\eta_{1:n}} \left[ \sup_{\phi\in \Phi_n} \frac{1}{n} \sum_{i=1}^n \eta_i \phi(X_i) \right] \lesssim \bE_{X_{1:n}} \left[ \cR_n(\Phi_n(X_{1:n})) \right] \log n \lesssim \frac{M_n \log n}{\sqrt{n}}.
\]
In summary, we conclude that
\begin{equation}\label{temp3}
\bE_{\cD_n} \left[\cL(f_n^*) - \cL_n(f_n^*) \right] \lesssim \frac{M_n \log n}{\sqrt{n}}.
\end{equation}

If $\cH = \cH^\alpha$ with $\alpha<(d+3)/2$, by Corollary \ref{app overp}, we have
\[
\sup_{h\in \cH^\alpha} \inf_{f\in \cF_n} \|h-f\|_{L^\infty(\bB^d)} \lesssim W_n^{-\frac{\alpha}{d}} \lor M_n^{-\frac{2\alpha}{d+3-2\alpha}}.
\]
Combining with (\ref{temp1}) and (\ref{temp3}), we know that if we choose $M_n \asymp n^{\frac{1}{2} - \frac{2\alpha}{d+3+2\alpha}}$ and $W_n  \gtrsim n^{\frac{d}{d+3+2\alpha}}$, then
\[
\bE_{\cD_n} \|f_n^*-h\|_{L^2(\mu)}^2 \lesssim W_n^{-\frac{2\alpha}{d}} \lor M_n^{-\frac{4\alpha}{d+3-2\alpha}} + \frac{M_n \log n}{\sqrt{n}} \lesssim n^{-\frac{2\alpha}{d+3+2\alpha}} \log n.
\]
Similarly, if $\cH = \cF_{\sigma}(1)$, by Corollary \ref{app overp}, then there exist a constant $M\ge 1$ such that, if $M_n=M$,
\[
\bE_{\cD_n} \|f_n^*-h\|_{L^2(\mu)}^2 \lesssim W_n^{-\frac{d+3}{d}} + \frac{M \log n}{\sqrt{n}}.
\]
Thus, for any $W_n \gtrsim n^{d/(2d+6)}$, we get $\bE_{\cD_n} \|f_n^*-h\|_{L^2(\mu)}^2 \lesssim n^{-1/2} \log n$,
\end{proof}

\subsection{Convolutional neural networks}

In contrast to the vast amount of theoretical studies on fully connected neural networks, there are only a few papers analyzing the performance of convolutional neural networks \cite{zhou2020theory,zhou2020universality,fang2020theory,petersen2020equivalence,lin2022universal,feng2021generalization,mao2023approximating,zhou2024learning}. The recent work \cite{lin2022universal} showed the universal consistency of CNNs for nonparametric regression. In this section, we show how to use our approximation results to analyze the convergence rates of CNNs.

Following \cite{zhou2020universality}, we introduce a sparse convolutional structure on deep neural networks. Let $s\ge 2$ be a fixed integer, which is used to control the filter length. Given a sequence $w=(w_i)_{i\in\bZ}$ on $\bZ$ supported on $\{0,1,\dots,s\}$, the convolution of the filter $w$ with another sequence $x=(x_i)_{i\in \bZ}$ supported on $\{1,\dots,d\}$ is a sequence $w*x$ given by 
\[
(w*x)_i := \sum_{j\in \bZ} w_{i-j} x_j = \sum_{j=1}^d w_{i-j} x_j, \quad i\in \bZ.
\]
Regarding $x$ as a vector of $\bR^d$, this convolution induces a $(d+s) \times d$ Toeplitz type convolutional matrix
\[
A^{w} := (w_{i-j})_{1\le i \le d+s, 1\le j\le d}.
\]
Note that the number of rows of $A^{w}$ is $s$ greater than the number of columns. This leads us to consider deep neural networks of the form $(\ref{deep NN})$ with linearly increasing widths $\{N_\ell = d + \ell s\}_{\ell =0}^L$. We denote by $\CNN(s,L)$ the set of functions that can be parameterized in the form $(\ref{deep NN})$ such that $A^{(\ell)} = A^{w^{(\ell)}}$ for some filter $w^{(\ell)}$ supported on $\{0,1,\dots,s\}$, $0\le \ell \le L-1$, and the biases $b^{(\ell)}$ take the special form
\begin{equation}\label{bias form}
b^{(\ell)} = \left(b^{(\ell)}_1,\dots, b^{(\ell)}_s, b^{(\ell)}_{s+1}, \dots, b^{(\ell)}_{s+1}, b^{(\ell)}_{N_\ell-s+1},\dots, b^{(\ell)}_{N_\ell} \right)^\intercal, \quad 0\le \ell \le L-2,
\end{equation}
with $N_\ell -2s$ repeated components in the middle. By definition, it is easy to see that $\CNN(s,L) \subseteq \NN(d+Ls,L)$. The assumption on the special form (\ref{bias form}) of biases is used to reduce the free parameters in the network. As in \cite{zhou2020universality}, one can compute that the number of free parameters in $\CNN(s,L)$ is $(5s+2)L+2d-2s$, which grows linearly on $L$.

The next proposition shows that all functions in $\NN(N,1)$ can be implemented by CNNs.

\begin{proposition}[\cite{zhou2020universality}]\label{rep shallow by cnn}
If $N,L\in\bN$ satisfy $L\ge \lfloor\frac{Nd}{s-1}+1 \rfloor$, then $\NN(N,1) \subseteq \CNN(s,L)$.
\end{proposition}
\begin{proof}
This result is proven in \cite[Proof of Theorem 2]{zhou2020universality}. We only give a sketch of the construction for completeness. Any function $f\in \NN(N,1)$ can be written as $f(x) = \sum_{i=1}^N c_i \sigma(a_i^\intercal x+b_i) + c_0$, where $a_i\in \bR^d$ and $b_i,c_i\in \bR$. Define a sequence $v$ supported on $\{0,\dots,Nd-1\}$ by stacking the vectors $a_1,\dots, a_N$ (with components reversed) by
\[
(v_{Nd-1},\dots,v_0) = (a_N^\intercal,\dots,a_1^\intercal).
\]
Applying \cite[Theorem 3]{zhou2020universality} to the sequence $v$, we can construct filters $\{w^{(\ell)}\}_{\ell=0}^{L-1}$ supported on $\{0,1,\dots,s\}$ such that $v=w^{(L-1)}*w^{(L-2)}*\dots *w^{(0)}$, which implies $A^{w^{(L-1)}} \cdots A^{w^{(0)}} = A^v \in \bR^{(d+Ls)\times d}$. Note that, by definition, for $i=1,\dots,N$, the $id$-th row of $A^v$ is exactly $a_i^\intercal$. Then, for $\ell=0,\dots,L-2$, we can choose $b^{(\ell)}$ satisfying (\ref{bias form}) such that $f^{(\ell+1)}(x) = A^{w^{(\ell)}} \cdots A^{w^{(0)}}x+B^{(\ell)}$, where $B^{(\ell)}>0$ is a sufficiently large constant that makes the components of $f^{(\ell+1)}(x)$ positive for all $x\in \bB^d$. Finally, we can construct $b^{(L-1)}$ such that $f^{(L)}_k(x)=\sigma(a_i^\intercal x+b_i)$ for $i=1,\dots,N$ and $k=id$, which implies $f\in \CNN(s,L)$.
\end{proof}

Note that Proposition \ref{rep shallow by cnn} shows each shallow neural network can be represent by a CNN, with the same order of number of parameters. As a corollary, we obtain approximation rates for CNNs.

\begin{corollary}\label{app cnn}
Let $s\ge 2$ be an integer. 
\begin{enumerate}[label=\textnormal{(\arabic*)},parsep=0pt]
\item For $\cH^\alpha$ with $0<\alpha<(d+3)/2$, we have
\[
\sup_{h\in \cH^\alpha} \inf_{f\in \CNN(s,L)} \|h-f\|_{L^\infty(\bB^d)} \lesssim L^{-\frac{\alpha}{d}}.
\]

\item For $\cF_{\sigma}(1)$, we have
\[
\sup_{h \in \cF_{\sigma}(1)} \inf_{f\in \CNN(s,L)} \|h-f\|_{L^\infty(\bB^d)} \lesssim L^{-\frac{1}{2} - \frac{3}{2d}}.
\]
\end{enumerate}
\end{corollary}
\begin{proof}
For any $N\in \bN$, we take $L=\lfloor\frac{Nd}{s-1}+1 \rfloor \asymp N$, then $\cF_{\sigma}(N,M) \subseteq \NN(N,1) \subseteq  \CNN(s,L)$ for any $M>0$, by Proposition \ref{rep shallow by cnn}. (1) follows from Corollary \ref{app neuron corollary} and (2) is from Lemma \ref{app neuron}.
\end{proof}

Since the number of parameters in $\CNN(s,L)$ is approximately $L$, the rate $\cO(L^{-\alpha/d})$ in part (1) of Corollary \ref{app cnn} is the same as the rate in \cite{yarotsky2017error} for fully connected neural networks. However, \cite{yarotsky2018optimal,lu2021deep} showed that this rate can be improved to $\cO(L^{-2\alpha/d})$ for fully connected neural networks by using the bit extraction technique \cite{bartlett2019nearly}. It would be interesting to see whether this rate also holds for $\CNN(s,L)$.

As in Theorem \ref{regression shallow}, we use Lemma \ref{gen bound by covering} to decompose the error and bound the approximation error by Corollary \ref{app cnn}. The covering number is bounded again by pseudo-dimension.

\begin{theorem}\label{regression cnn}
Let $f_n^*$ be the estimator (\ref{ERM}) with $\cF_n = \CNN(s,L_n)$, where $s\ge 2$ is a fixed integer, and set $B_n = c_1\log n$ for some constant $c_1>0$. 
\begin{enumerate}[label=\textnormal{(\arabic*)},parsep=0pt]
\item If $\cH = \cH^\alpha$ with $\alpha<(d+3)/2$, we choose
\[
L_n \asymp n^{\frac{d}{2d+2\alpha}}, 
\]
then
\[
\bE_{\cD_n} \|\cT_{B_n}f_n^*-h\|_{L^2(\mu)}^2 \lesssim n^{-\frac{\alpha}{d+\alpha}} (\log n)^4.
\]

\item If $\cH = \cF_{\sigma}(1)$, we choose
\[
L_n \asymp n^{\frac{d}{3d+3}},
\]
then
\[
\bE_{\cD_n} \|\cT_{B_n}f_n^*-h\|_{L^2(\mu)}^2 \lesssim  n^{-\frac{d+3}{3d+3}} (\log n)^4.
\]
\end{enumerate}
\end{theorem}
\begin{proof}
The proof is the same as Theorem \ref{regression shallow} and \cite{zhou2024learning}. We can use (\ref{cover bound}) to bound the covering number by the pseudo-dimension. For convolutional neural networks, \cite{bartlett2019nearly} gave the following estimate of the pseudo-dimension:
\[
\Pdim(\cT_{B_n}\cF_n) \lesssim L_n p(s,L_n) \log(q(s,L_n)) \lesssim L_n^2 \log L_n,
\]
where $p(s,L_n)=(5s+2)L_n+2d-2s \lesssim L_n$ and $q(s,L_n)\le L_n (d+sL_n) \lesssim L_n^2$ are the numbers of parameters and neurons of the network $\CNN(s,L_n)$, respectively. Therefore,
\[
\log \cN(\epsilon, \cT_{B_n}\cF_n,\|\cdot\|_{L^1(X_{1:n})}) \lesssim L_n^2 \log (L_n) \log(B_n/\epsilon).
\] 

Applying Lemma \ref{gen bound by covering} and Corollary \ref{app cnn}, if $\cH= \cH^\alpha$ with $\alpha<(d+3)/2$, then
\[
\bE_{\cD_n} \|\cT_{B_n}f_n^*-h\|_{L^2(\mu)}^2 \lesssim \frac{L_n^2 \log (L_n) (\log n)^3}{n} + L_n^{-\frac{2\alpha}{d}}.
\]
We choose $L_n \asymp n^{d/(2d+2\alpha)}$, then $\bE_{\cD_n} \|\cT_{B_n}f_n^*-h\|_{L^2(\mu)}^2 \lesssim n^{-\alpha/(d+\alpha)} (\log n)^4$. Similarly, if $\cH= \cF_{\sigma}(1)$, then
\[
\bE_{\cD_n} \|\cT_{B_n}f_n^*-h\|_{L^2(\mu)}^2 \lesssim \frac{L_n^2 \log (L_n) (\log n)^3}{n} + L_n^{-\frac{d+3}{d}}.
\]
We choose $L_n \asymp n^{d/(3d+3)}$, then $\bE_{\cD_n} \|\cT_{B_n}f_n^*-h\|_{L^2(\mu)}^2 \lesssim n^{-(d+3)/(3d+3)} (\log n)^4$.
\end{proof}

Finally, we note that the recent paper \cite{zhou2024learning} also studied the convergence of CNNs and proved the rate $\cO(n^{-1/3}(\log n)^2)$ for $\cH^\alpha$ with $\alpha>(d+4)/2$. The convergence rate we obtained in Theorem \ref{regression cnn} for $\cF_{\sigma}(1)$, which includes $\cH^\alpha$ with $\alpha>(d+3)/2$ by Theorem \ref{app norm}, is slightly better than their rate.

\section{Conclusion}\label{sec: conclusion}

This paper has established approximation bounds for shallow ReLU$^k$ neural networks. We showed how to use these bounds to derive approximation rates for (deep or shallow) neural networks with constraints on the weights and convolutional neural networks. We also applied the approximation results to study the convergence rates of nonparametric regression using neural networks. In particular, we established the optimal convergence rates for shallow neural networks and showed that over-parameterized neural networks can achieve nearly optimal rates. 

There are a few interesting questions we would like to propose for future research. First, for approximation by shallow neural networks, we establish the optimal rate in the supremum norm by using the results of \cite{siegel2023optimal} (Lemma \ref{app neuron}). The paper \cite{siegel2023optimal} actually showed that approximation bounds similar to Lemma \ref{app neuron} also hold in Sobolev norms. We think it is a promising direction to extend our approximation results in the supremum norm (Theorem \ref{app norm} and Corollary \ref{app neuron corollary}) to the Sobolev norms. Second, it is unclear whether over-parameterized neural networks can achieve the optimal rate for learning functions in $\cH^\alpha$. It seems that refined generalization error analysis is needed. Finally, it would be interesting to extend the theory developed in this paper to general activation functions and study how the results are affected by the activation functions.

\section*{Acknowledgments}

The work described in this paper was partially supported by InnoHK initiative, The Government of the HKSAR, Laboratory for AI-Powered Financial Technologies, the Research Grants Council of Hong Kong [Projects No. CityU 11306220 and 11308020] and National Natural Science Foundation of China [Project No. 12371103] when the second author worked at City University of Hong Kong. We thank the referees for their helpful comments and suggestions on the paper.

\bibliographystyle{plain}
\bibliography{Ref}
\end{document}